\documentclass{article}
\usepackage[accepted]{icml2021}

\icmltitlerunning{Latent Optimal Transport}

\usepackage{xr-hyper}
\usepackage{hyperref}       
\usepackage[T1]{fontenc}    
\usepackage{url}            
\usepackage{booktabs}       
\usepackage{amsfonts}       
\usepackage{nicefrac}       
\usepackage{microtype}      
\usepackage{mathtools}
\usepackage[utf8]{inputenc}
\usepackage[english]{babel}     

\renewcommand{\algorithmicinput}{\textbf{Input:}}
\renewcommand{\algorithmicoutput}{\textbf{Output:}}
\renewcommand{\algorithmicensure}{\textbf{Initialize:}}
\renewcommand{\algorithmicrequire}{\textbf{Return:}}

\usepackage{adjustbox}
     
\usepackage{microtype}
\usepackage{graphicx}

\newcommand{\alg}{\texttt{LOT}~}
\newcommand{\kot}{\texttt{KOT}~}
\newcommand{\ot}{OT~}
\newcommand{\fc}{\texttt{FC}~}
\newcommand{\lotwa}{\alg\hspace{-1mm}\texttt{-WA}~}
\newcommand{\lot}{\alg\hspace{-1mm}-\texttt{L2}~}

\usepackage{color}

\usepackage{amsmath}
\usepackage{amssymb}
\usepackage{amsthm}
\usepackage{wrapfig}
\usepackage{comment}
\usepackage{graphicx}
\usepackage{bbm}
\usepackage{url}
\usepackage{xr}

\usepackage{caption}
\usepackage{subcaption}
\usepackage{mathtools}
\usepackage{hyperref}
\usepackage{duckuments}

\usepackage{mathtools}

\usepackage{graphicx}

\def\argmin{\mathop{\textnormal{argmin}}}
\def\argmax{\mathop{\textnormal{argmax}}}

\newcommand{\x}{\mathbf{x}}
\newcommand{\y}{\mathbf{y}}
\newcommand{\z}{\mathbf{z}}
\makeatletter
\newtheorem*{rep@theorem}{\rep@title}
\newcommand{\newreptheorem}[2]{%
\newenvironment{rep#1}[1]{%
 \def\rep@title{#2 \ref{##1}}%
 \begin{rep@theorem}}%
 {\end{rep@theorem}}}
\makeatother

\newtheorem{proposition}{Proposition}
\newtheorem{lemma}{Lemma}
\newtheorem{corollary}{Corollary}

\newreptheorem{theorem}{Theorem}
\newtheorem{definition}{Definition}

\newtheorem{remark}{Remark}
\newreptheorem{lemma}{Lemma}
\newreptheorem{proposition}{Proposition}
\usepackage{amsmath}
\usepackage{soul} 
\definecolor{mydarkblue}{rgb}{0,0.08,0.45}
  \hypersetup{ %
    pdftitle={},
    pdfauthor={},
    pdfsubject={Proceedings of the International Conference on Machine Learning 2021},
    pdfkeywords={},
    pdfborder=0 0 0,
    pdfpagemode=UseNone,
    colorlinks=true,
    linkcolor=mydarkblue,
    citecolor=mydarkblue,
    filecolor=mydarkblue,
    urlcolor=black,
    pdfview=FitH}

\begin{document}  

\twocolumn[
\icmltitle{Making transport more robust and interpretable by moving data\\ through a small number of anchor points}




\begin{icmlauthorlist}
\icmlauthor{ Chi-Heng Lin}{1}
\icmlauthor{Mehdi Azabou}{1,2}
\icmlauthor{Eva L. Dyer}{1,2,3}
\end{icmlauthorlist}

\icmlaffiliation{1}{Department of Electrical and Computer
Engineering, Georgia Tech, Atlanta, Georgia, USA.}
\icmlaffiliation{2}{Machine Learning Program,
    Georgia Tech, Atlanta, Georgia, USA}
\icmlaffiliation{3}{Coulter Department of Biomedical Engineering,
    Georgia Tech \& Emory University, Atlanta, Georgia, USA}
\icmlcorrespondingauthor{Chi-Heng Lin}{cl3385@gatech.edu}
\icmlcorrespondingauthor{Mehdi Azabou}{mazabou@gatech.edu}
\icmlcorrespondingauthor{Eva L. Dyer}{evadyer@gatech.edu}

\icmlkeywords{Machine Learning, ICML}

\vskip 0.3in
]



\printAffiliationsAndNotice{} 
\begin{abstract}
Optimal transport (OT) is a widely used technique for distribution alignment, with applications throughout the machine learning, graphics, and vision communities. Without any additional structural assumptions on transport, however, \ot can be fragile to outliers or noise, especially in high dimensions.
Here, we introduce Latent Optimal Transport (\alg\hspace{-1mm}), a new approach for OT that simultaneously learns low-dimensional structure in data while leveraging this structure to solve the alignment task. The idea behind our approach is to learn two sets of ``anchors'' that constrain the flow of transport between a source and target distribution.
In both theoretical and empirical studies, we show that \alg regularizes the rank of transport and makes it more robust to outliers and the sampling density. We show that by allowing the source and target to have different anchors, and using \alg to align the latent spaces between anchors, the resulting transport plan has better structural interpretability and  highlights connections between both the individual data points and the local geometry of the datasets.
\end{abstract}

\section{Introduction}\label{sec:intro}
Optimal transport (OT) \cite{villani2008optimal} is a widely used technique for distribution alignment that learns a {\em transport plan} which moves mass from one distribution to match another. 
With recent advances in tools for regularizing and speeding up OT \cite{cuturi2013sinkhorn}, this approach has found  applications in many diverse areas of machine learning, including domain adaptation \cite{courty2014domain,flamary2016optimal}, generative modeling \cite{martin2017wasserstein,tolstikhin2017wasserstein}, 
document retrieval \cite{kusner2015word}, computer graphics \cite{solomon2014earth,solomon2015convolutional,bonneel2016wasserstein}, and computational neuroscience \cite{gramfort2015fast,lee2019hierarchical}. 

While the ground metric in \ot can be used to impose geometric structure into transport, without any additional 
assumptions, \ot can be fragile to outliers or noise, especially in high dimensions. To overcome this issue, additional structure, either in the data or in the transport plan, can be used to improve alignment or make transport more robust. Examples of methods that incorporate additional structure into OT include approaches that leverage hierarchical structure or cluster consistency \cite{lee2019hierarchical,yurochkin2019hierarchical,xu2020learning}, partial class information \cite{flamary2016optimal,courty2014domain}, submodular cost functions \cite{alvarez2018structured}, and low-rank constraints on the transport plan \cite{forrow2019statistical, altschuler2019massively}. Because of the difficulty of incorporating structure into OT, many of these methods need low-dimensional structure in data to be specified in advance (e.g., estimated clusters or labels).  

To simultaneously learn low-dimensional structure and use it to constrain transport, \citet{forrow2019statistical} recently introduced a statistical approach for OT that builds a factorization of the transport plan to regularize its rank. After factorization, transport from a source to target distribution can be visualized as the flow of mass through a small number of anchors (hubs), which serve as relay stations through which transportation must pass (see Figure~\ref{fig:motivation}, a vs.~b). 
Although this idea of moving data through anchors is appealing, in previous work, the anchors used to constrain transport are {\em shared} by the source and target. As a result, when the source and target contain different structures or experience domain shift \cite{courty2014domain}, shared anchors may not provide an adequate representation for both domains simultaneously.

In this work, we propose a new structured transport approach called Latent Optimal Transport (\alg\hspace{-1mm}). The main idea behind \alg is to factorize the transport plan into three components, where mass is moved: (i) from individual source points to source anchors, (ii) from the source anchors to target anchors, and (iii) from target anchors to individual target points (Figure~\ref{fig:motivation}c-d).
The intermediate transport plan captures the high-level structural similarity between the source and target, while the outer transport plans cluster data in their respective spaces.
In both theoretical and empirical studies, we show that \alg regularizes the rank of transport and has the effect of denoising the transport plan, making it more robust to outliers and sampling.  By allowing the source and target to have different anchors and aligning the latent spaces of the anchors, we show that the mapping between datasets can be more easily interpreted.

\begin{figure*}
     \centering

     \subfloat[][\footnotesize{\ot}]{\includegraphics[width=.24\textwidth]{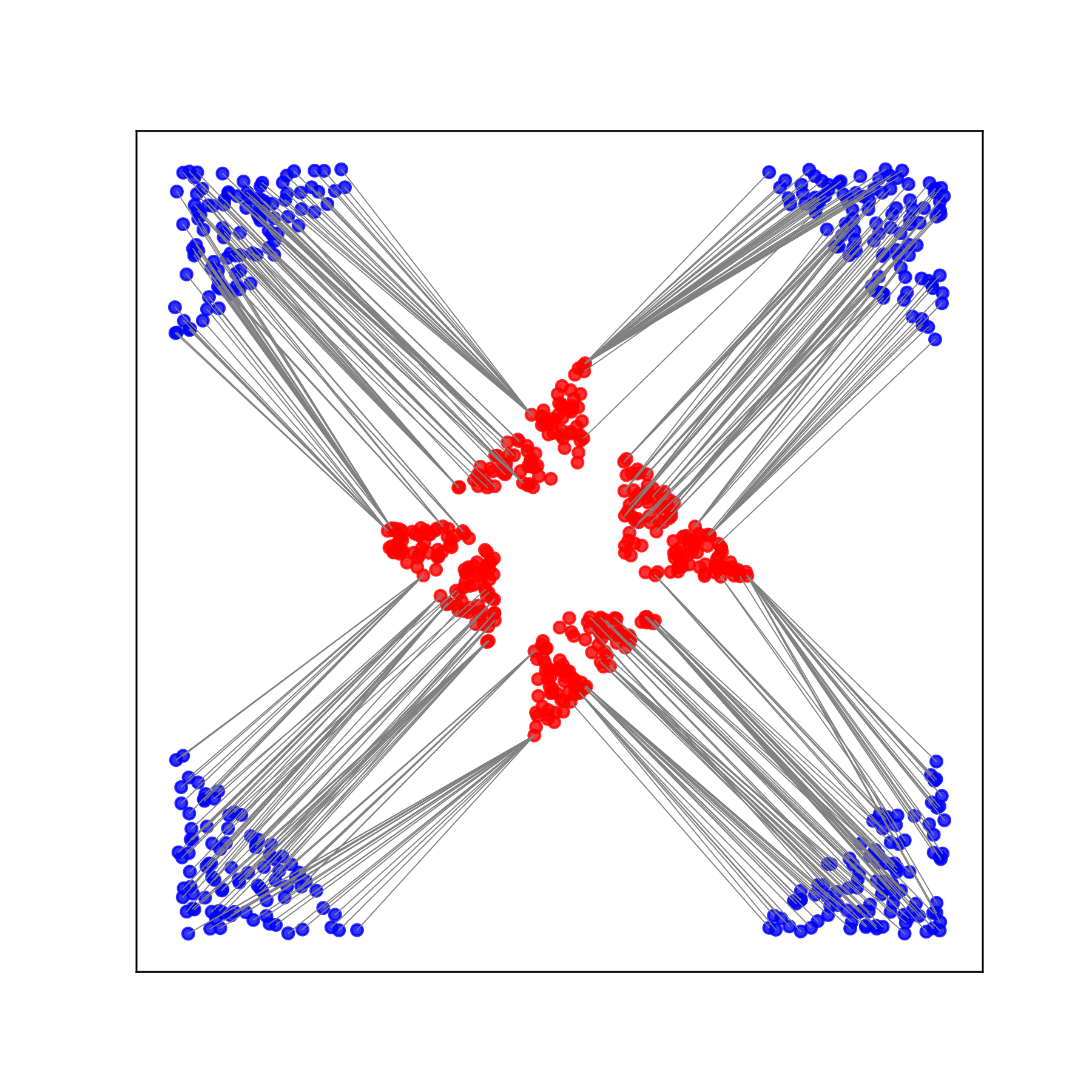}\label{fig:illus1}\vspace{-3mm}}
     \subfloat[][\footnotesize{\fc}]{\includegraphics[width=.24\textwidth]{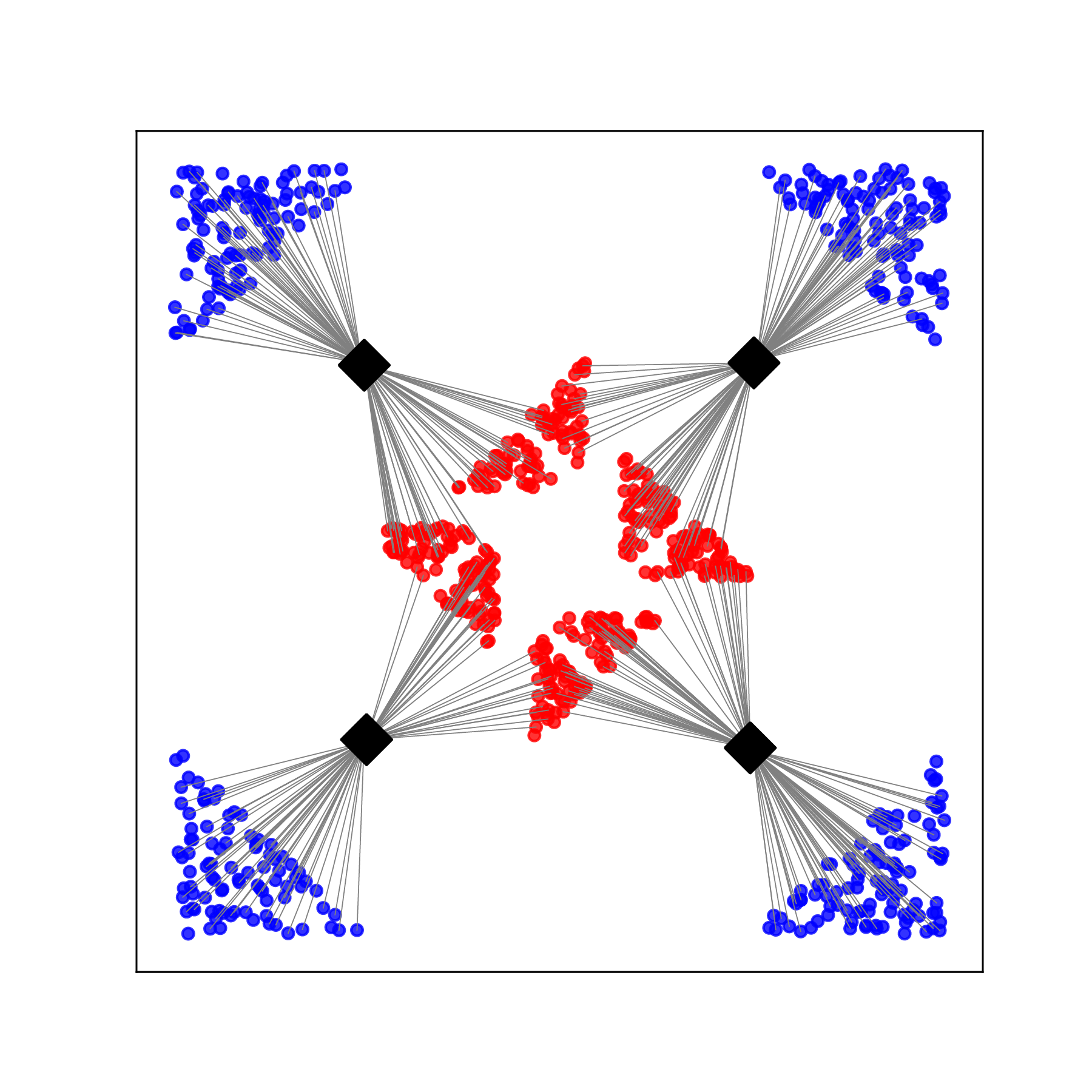}\vspace{-3mm}}
     \subfloat[][\footnotesize{\alg(4,4)}]{\includegraphics[width=.24\textwidth]{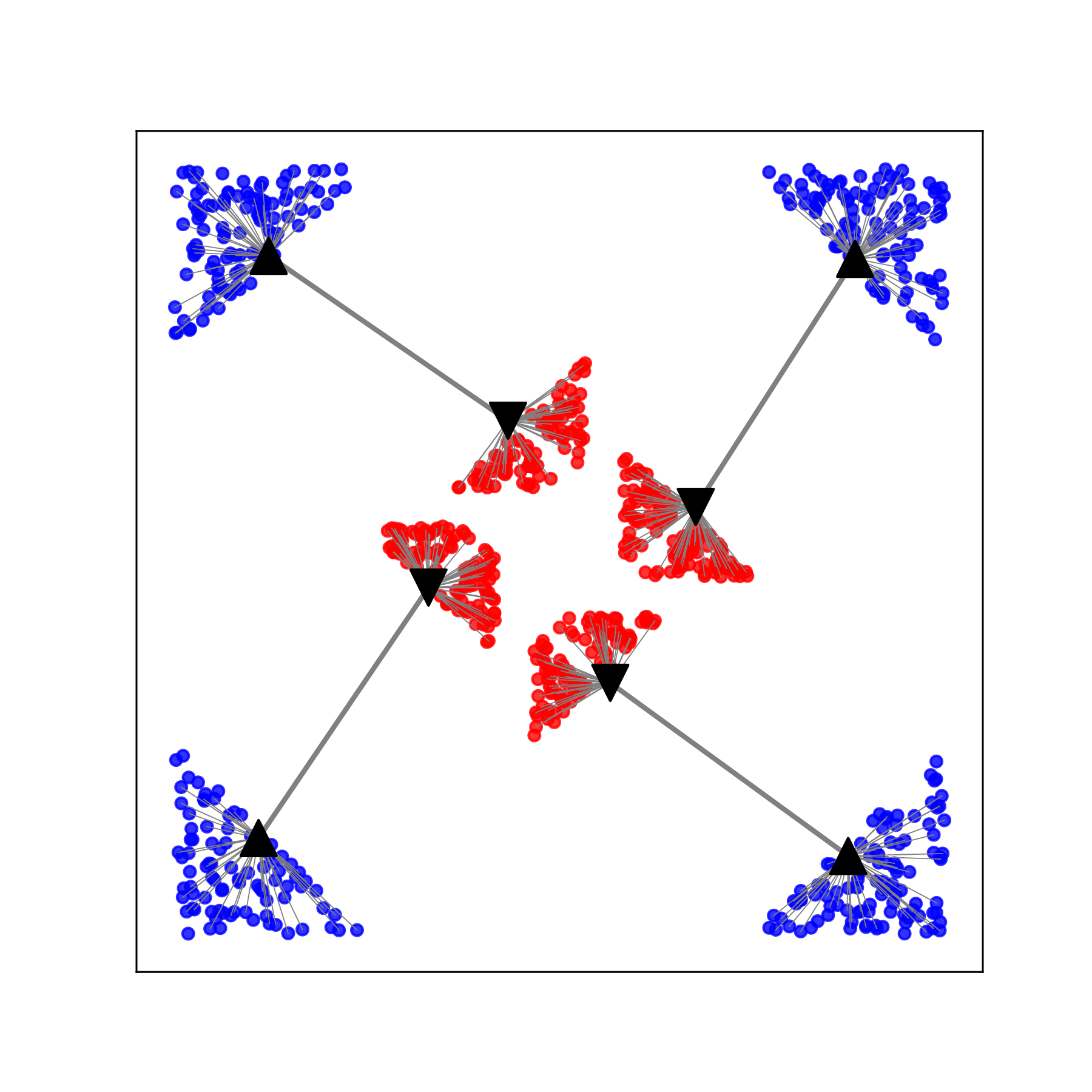}\label{fig:illus2}\vspace{-3mm}}
     \subfloat[][\footnotesize{\alg(8,4)}]{\includegraphics[width=.24\textwidth]{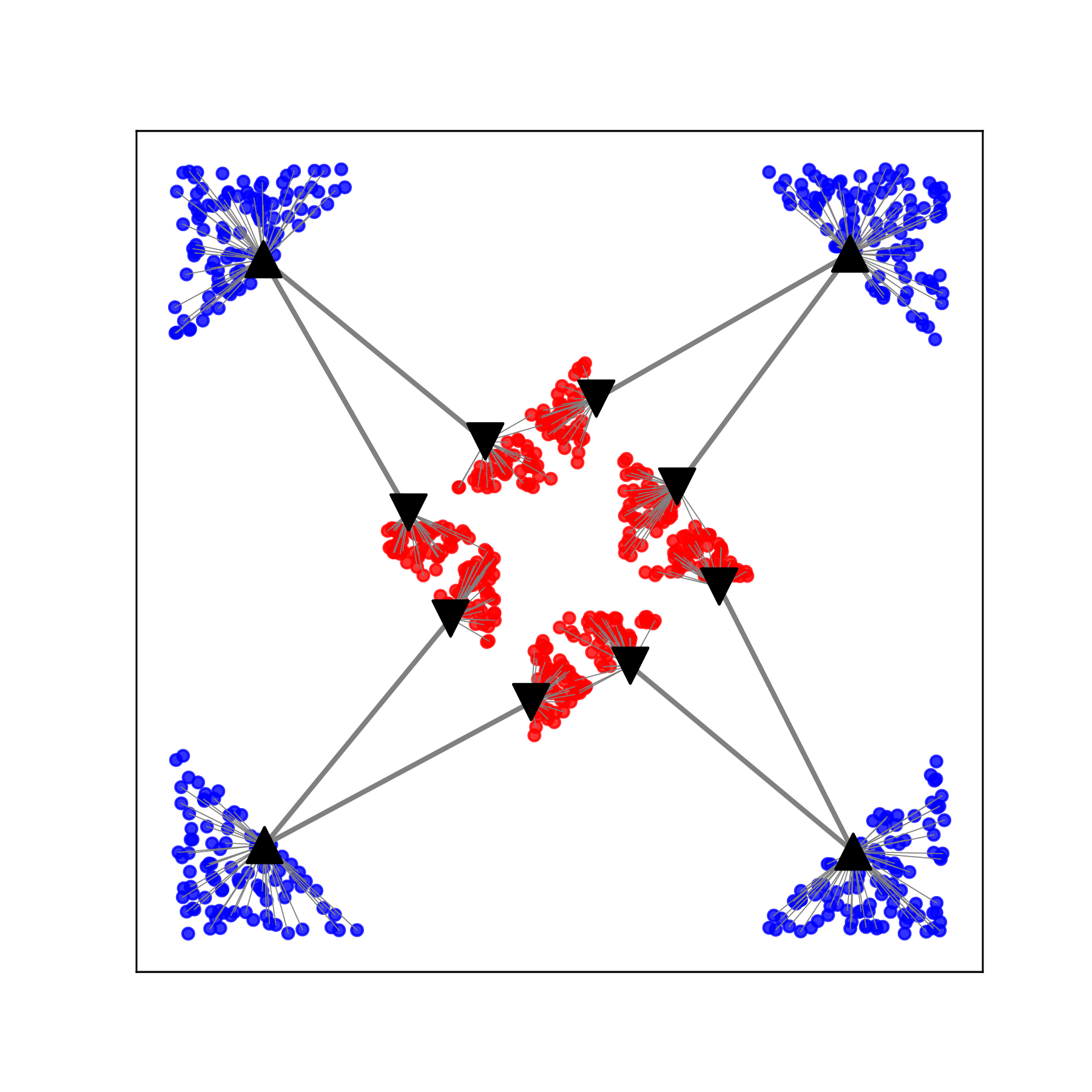}\label{fig:illus3}\vspace{-3mm}}
     \caption{\footnotesize{{\bf \em  Comparisons of transport plans obtained for different methods applied to clustered data after domain shift}. Here, we visualize the connection between the source (blue) $x$  and its estimated target (red) $\hat{y}=\argmax_y p(y|x)$. From left to right, we show the standard \ot plan (a) and the factored coupling (\url{FC}) approach (b). To the right, we show the result of \alg when we use 4 anchors in the target with the same number in the target (c) and 8 anchors in the target (d).}
     \label{fig:motivation}\vspace{-4mm}}
\end{figure*}

Specifically, our contributions are as follows. (i) We introduce \alg\hspace{-1mm}, a new form of structured transport, and propose an efficient algorithm that solves our proposed objective (Section \ref{sec:formulation}), (ii)
Theoretically, we show that \alg can be interpreted as a relaxation to OT,
and from a statistical point-of-view, it overcomes the curse of dimensionality in terms of the sampling rate (Section \ref{sec:analysis}),
(iii) We study the robustness of the approach to noise, sampling, and various data perturbations when applied to both synthetic data and domain adaptation problems in computer vision  (Section~\ref{sec:experiment}).

\section{Background}\label{sec:Background}
\paragraph{Optimal Transport:}
\label{sec:ot}
Optimal transport (OT) \cite{villani2008optimal,santambrogio2015optimal, peyre2019computational} is a distribution alignment technique that learns a transport plan that specifies how to move mass from one distribution to match another.  Specifically, 
consider two sets of data points encoded in matrices, the {\em source} $\mathbf{X} = [\x_1,\dots, \x_n]$ and the {\em target} $\mathbf{Y}=[ \y_1,\dots, \y_m]$, where $\x_i \in {\mathcal{X}}$, $\y_j \in {\mathcal{Y}}$, $\forall i,j$.
Assume they are endowed with discrete measures
$\mu = \sum_{i=1}^Np(\x_i)\delta_{\x_i}$, $\nu = \sum_{j=1}^Mp(\y_j)\delta_{\y_j}$, respectively.
The cost of transporting $\x_i$ to $\y_j$ is $c( \x_i, \y_j)$, where $c$ denotes some cost function. \ot considers the most cost-efficient transport by solving the following problem:\footnote{The problem can be generalized to setting of continuous measures by OT$_c(\mu,\nu)=\min_{\gamma\in{\cal{G}}}\int_{{\cal{X}}\times{\cal{Y}}}c(x,y)d\gamma(x,y)$, ${\cal{G}}=\{\gamma:\int_{\cal{Y}} d\gamma(x,y)=\mu,\int_{\cal{X}} d\gamma(x,y)=\nu\}$.}
\begin{align}\label{ot}
   \text{OT}_{\mathbf{C}}(\mu,\nu):= \min _{{ \mathbf{P}\mathbf{1}} = \mu, { \mathbf{P}}^T{\mathbf{1}} =\nu}\langle {\mathbf{C}}, { \mathbf{P}}\rangle,
\end{align}
where $\mathbf{P}:=[p(\x_i,\y_j)]_{i,j}$ is the source-to-target transport plan matrix (coupling), and ${\mathbf{C}}=[c( \x_i, \y_j)]_{i,j}$ is the cost matrix. When $c(\x , \y)=d( \x, \y)^p$, where $d$ is a distance function, ${\cal{W}}_p:=\text{OT}_{\textbf{C}}^{{1}/{p}}$ defines a distance called the $p$-Wasserstein distance. The objective in (\ref{ot}) is a linear programming problem, where computation speed can be prohibitive if $n$ is large \cite{pele2009fast}. A common speedup is to replace the objective by an entropy-regularized proxy, 
\begin{align}
 \label{regot}
   \text{OT}_{\mathbf{C},\varepsilon}(\mu,\nu):&= \hspace{-2mm} \min _{{ \mathbf{P}\mathbf{1}}  = \mu, { \mathbf{P}}^T{\mathbf{1}} =\nu}\langle \mathbf{C}, {\mathbf{P}}\rangle -\varepsilon\mathbf{H}(\mathbf{P})\nonumber\\
   &=\min _{{ \mathbf{P}\mathbf{1}}  = \mu, { \mathbf{P}}^T{\mathbf{1}} =\nu}\varepsilon\text{KL}({\bf P}||{\bf K}),
\end{align}
where ${\bf K}$ is the Gibbs kernel induced by the element-wise exponential of the cost matrix $\mathbf{K} := \exp(-\mathbf{C}/\varepsilon)$, $\mathbf{H}(\mathbf{P}):=-\sum_{ij}\mathbf{P}_{ij}\log(\mathbf{P}_{i,j})$ is the Shannon entropy, and $\varepsilon$ is a user-specified hyperparameter that controls the amount of entropic regularization that is introduced. We can alternatively write the objective function as a minimization of $\varepsilon \text{KL}(\mathbf{P}\|\mathbf{K})$, where KL denotes the Kullback-Leibler divergence.
In practice, the entropy-regularized form is often used over the original objective (\ref{ot}) as it admits a fast method called the Sinkhorn algorithm \cite{cuturi2013sinkhorn,altschuler2017near}.
Hence, we will use \ot to refer to the entropy-regularized form unless specified otherwise in the context.

{\bf Optimal Transport via Factored Couplings:}~
\label{sec:fc}
Factored Coupling (\url{FC}) is proposed in \cite{forrow2019statistical} to reduce the sample complexity of OT in high dimensions. Specifically, it adds an additional constraint to (\ref{ot}) by enforcing the transport plan to be of the following factored form,
\begin{align}\label{fc}
\vspace{-1mm}
        p(\x_i,\y_j) = \sum_{l=1}^k p(\z_l)p(\x_i|\z_l)p(\y_j|\z_l).
        \vspace{-2mm}
\end{align} 
This has a nice interpretation: $\z_l$ serves as a common ``anchor'' that transportation from $\x_i$ to $\y_j$ must pass through. It turns out that \url{FC} is closely related to the Wasserstein barycenter problem \cite{agueh2011barycenters,cuturi2014fast,cuturi2016smoothed}, $\min_\nu \sum_{i=1}^N{\cal{W}}^2_2(\mu_i,\nu)$,
where $\nu$ is the Procrustes mean to distributions $\mu_i$, $i=1,\dots,N$ with respect to the squared 2-Wasserstein distance. A crucial insight from \cite{forrow2019statistical} is that for $N=2$, the barycenter $\nu$ could approximate the optimal anchors to a transport plan of the form (\ref{fc}) that minimizes the objective in (\ref{ot}).

\section{Latent Optimal Transport }
\label{sec:formulation}

\subsection{Motivation}

Most datasets have low-dimensional latent structure, but \ot does not naturally use it during transport. This motivates the idea that distribution alignment methods should both {\em reveal} the latent structure in the data in addition to  aligning these latent structures. An illustrative example is provided in Figure~\ref{fig:motivation}; here, we show the transport plan for a source (red points) and a target (blue points), both of which exhibit clear cluster structures. Because \ot transports points independently, the points can be easily mapped outside of their original cluster (a).
In comparison, low-rank OTs (b-d) induce transport plans that are better at preserving clusters. 
In (b), because factored coupling (\url{FC}) transports points via common anchors (black squares), the anchors need to interpolate between both distributions, and it loses the freedom of choosing different structures for the source and target.
On the other hand, by specifying different numbers of anchors for the source and target individually (c vs.~d), \alg can 
extract different structures and output different transport plans.

\subsection{Problem formulation}
Consider data matrices $\mathbf{X}$ and $\mathbf{Y}$ and their measures $\mu$, $\nu$,  as detailed in Section~\ref{sec:Background}. 
We introduce  ``anchors'' through which points must flow, thus constraining the transportation. The anchors are stacked in data matrices $\mathbf{Z}_x:=[\mathbf{z}^x_1,\dots,\mathbf{z}^x_{k_x}]$,  $\mathbf{Z}_y:=[\mathbf{z}^y_1,\dots,\mathbf{z}^y_{k_y}]$.
We denote the measures of the source and target anchors as
$\mu_z = \sum_{m=1}^{k_x}p(\z^x_m)\delta_{\z^x_m}$ and $\nu_z = \sum_{n=1}^{k_y}p(\z^y_n)\delta_{\z^y_n}$.
For any set ${\cal{A}}$, we further denote $\Delta^k_{\cal{A}}:=\left\{\sum_{i=1}^k \mathbf{\omega}_i\delta_{\mathbf{a}_i}:\sum_{i=1}^k\mathbf{\omega}_i =1,\mathbf{\omega}_i\geq 0,\mathbf{a}_i\in {\cal{A}},\forall i\right\}$ as the set of probability measures on ${\cal{A}}$ that has discrete support of size up to $k$. Hence $\mu_z \in \Delta^{k_x}_{{\mathcal{Z}}_x}$, $\nu_z\in \Delta^{k_y}_{{\mathcal{Z}}_y}$, where ${\mathcal{Z}}_x$ (resp. ${\mathcal{Z}}_y$) is the space of source (resp. target) anchors.
If we interpret the conditional probability $p(a|b)$ as the strength of transportation from $b$ to $a$, then, using the chain rule, the concurrence probability $p(\x_i,\y_j)$ of $\x_i$ and $\y_j$ can be written as, 
\begin{align}\label{transportform}
    p(\x_i, \y_j) =& \sum_{m,n} p(\x_i)p(\z_m^x|\x_i)p(\z^y_n | \z_m^x)p(\y_j| \z^y_n)\nonumber\\
    =&\sum_{m,n} p(\x_i,\z_m^x)\frac{p(\z_m^x,\z^y_n)}{p(\z_m^x)p(\z^y_n)}p(\z^y_n,\y_j).
\end{align}
When encoding these probabilities using a transport matrix $\mathbf{P}:=[p(\x_i,\y_j)]_{i,j}$, the factorized form (\ref{transportform}) can be written as,
\begin{align}\label{form}
\vspace{-1mm}
    \mathbf{P} = \mathbf{P}_x\text{diag}(\mathbf{u}_z^{-1}) \mathbf{P}_z \text{diag}(\mathbf{v}_z^{-1})\mathbf{P}_y,
\end{align} 
where $\mathbf{P}_x$ encodes transport from source space to source anchor space (i.e., $p(\x_i,\z^x_m)$), $\mathbf{P}_z$ encodes transport from source anchor space to target anchor space, $\mathbf{P}_y$ encodes transport from target anchor space to target space
, and
 $\mathbf{u}_z:=[p(\z^x_1),\cdots,p(\z^{x}_{k_x})]$, $\mathbf{v}_z:=[p(\z^y_1),\cdots,p(\z^y_{k_y})]$ encode the latent distributions of anchors. 
To learn each of these transport plans, we must first designate the ground metric used to define the cost in each of the three stages. 
The cost matrices $\mathbf{C}_x,\mathbf{C}_y$ determine how points will be transported to their respective anchors and thus dictate how the data structure will be extracted. 
We will elaborate on the choice of costs in Section~\ref{sec:costs}.

We now formalize our proposed approach to transport in the following definition.
\begin{definition}
\label{def1}
Let $\mathbf{C}_x$, $\mathbf{C}_y$ denote the cost matrices between the source/target and their representative anchors, and let $\mathbf{C}_z$ denote the cost matrix between anchors. We define the latent optimal transport (\alg\hspace{-.5mm}) problem as, 
\begin{align*}
\rm{OT}^{{L}}(\mu,\nu):&=\inf_{\mu_z \in \Delta_{{\mathcal{Z}}_x}^{k_x},\nu_z\in \Delta_{{\mathcal{Z}}_y}^{k_y}}\big\{\rm{OT}_{\mathbf{C}_x}(\mu,\mu_z)\nonumber\\+ &\rm{OT}_{\mathbf{C}_z}(\mu_z,\nu_z)+\rm{OT}_{\mathbf{C}_y}(\nu_z,\nu)\big\}, 
\vspace{-3mm}
\end{align*}
where ${\mathcal{Z}}_x$ and ${\mathcal{Z}}_y$ are the latent spaces of the source and target anchors, respectively.\footnotemark 
\vspace{-2mm}
\end{definition}
\footnotetext{
This definition extends naturally to continuous measures by replacing cost matrix $\mathbf{C}$ with cost function $c$.}
The intuition behind Def. \ref{def1} is that we use $\rm{OT}_{\mathbf{C}_x}(\mu,\mu_z)$ and $\rm{OT}_{\mathbf{C}_y}(\nu_z,\nu)$ to capture group structure in each space, and then $\rm{OT}_{\mathbf{C}_z}(\mu_z,\nu_z)$ to align the source and target by determining the transportation across anchors. Hence, \alg can be interpreted as an optimization of joint clustering and alignment.
The flexibility of cost matrices allows \alg to capture different structures and induce different transport plans. In Section~\ref{sec:analysis}, we further show that \alg can be regarded as a relaxation of an \ot problem. 
\begin{remark}
In \citet{forrow2019statistical}, the authors introduce the notion of the transport rank for a transport plan $\mathbf{P}$ as the minimum number of product probability measures that its corresponding coupling can be composed from, i.e., $p(\x,\y) = \sum_{i=1}^r\lambda_i\left(p_i(\x)\otimes p_i(\y)\right)$, $\lambda_i\geq 0$, $\forall i$. In general, given a transportation plan $\mathbf{P}$, the transport rank ${\rm rank}_+(\mathbf{P})$ is lower bounded by its usual matrix rank ${\rm rank}(\mathbf{P})$. 
In the case of \url{LOT}, the transport plan induced by Def. \ref{def1} satisfies ${\rm rank}(\mathbf{P})\leq {\rm rank}_+(\mathbf{P})\leq \min(k_x,k_y)$. Thus, by selecting a small number of anchors we naturally induce a low-rank solution for transport.
\end{remark}


Next, we show some properties of \alg that highlight its similarity to a metric.
\begin{proposition}\label{prop1}
Suppose the latent spaces ${\cal{Z}}_x={\cal{Z}}_y$ are the same as the original data spaces ${\cal{X}}={\cal{Y}}$, and the cost matrices are defined by $\mathbf{C}_x[a,b] = \mathbf{C}_z[a,b]= \mathbf{C}_y[a,b]= d(a,b)^p$, where $p\geq1$ and $d$ is some distance function.
If we define the latent Wasserstein discrepancy as ${\cal{W}}^L_p:={(\rm{OT}^L)}^{1/p}$, then there exist $\kappa>0$ such that, for any $\mu$, $\nu $ and $\zeta$ having latent distributions of support sizes up to $k$, the discrepancy satisfies,
\vspace{-1mm}
\begin{itemize}
\vspace{-1mm}
    \item ${\cal{W}}_p^L(\mu,\nu)\geq 0$
    \vspace{-1mm}
    \item ${\cal{W}}_p^L(\mu,\nu) = {\cal{W}}_p^L(\nu,\mu)$ 
    \vspace{-1mm}
    \item ${\cal{W}}_p^L(\mu,\nu) \leq \kappa \left({\cal{W}}_p^L(\mu,\zeta)+{\cal{W}}_p^L(\zeta,\nu)\right)$ 
\end{itemize}
\vspace{-2mm}
\end{proposition}

The low-rank nature of \alg has a biasing effect that results in ${\cal{W}}_p^L(\mu,\mu)>0$ for a general $\mu$. 
We can debias it by defining its variant $\tilde{{\cal{W}}}_p^L(\mu,\nu):=\left(\left({\cal{W}}_p^L(\mu,\nu)\right)^p-\min\limits_{z_k\in{\Phi}_x}{\cal{W}}_p^p(\mu,z_k)-\min\limits_{z'_k\in{\Phi}_y}{\cal{W}}_p^p(\nu,z'_k)\right)^{1/p}$,
where ${{\Phi}}_x=\Delta_{{\cal{Z}}_x}^{k_x}$, ${{\Phi}}_y=\Delta_{{\cal{Z}}_y}^{k_y}$.
The following property connects $\tilde{{\cal{W}}}_p^L(\mu,\nu)$ to k-means clustering.
\begin{corollary}\label{cor:kmean}
Under the assumptions of Proposition \ref{prop1}, if $p=2$ and $k_x=k_y=k$, then $\forall \mu,\nu$, we have $\tilde{{\cal{W}}}_2^L(\mu,\nu)\geq 0$. Furthermore,  $\tilde{{\cal{W}}}_2^L(\mu,\nu)>0$ if their k-means centroids or sizes of their k-means clusters differ.
\end{corollary}

\subsection{Establishing a ground metric}

\label{sec:costs}
In what follows, we will focus on the Euclidean space ${\cal{X}}={\cal{Y}}=\mathbb{R}^d$.
Instead of considering every source-to-target distance to build our transportation cost, we can use anchors as proxies for each point. A well-established way of encoding the distance that each point needs to travel to get to its nearest anchor, is to define the cost as:
\begin{align}
\label{eq:mahanorm}
   \mathbf{C}_x= d_{\mathbf{M}_x},  \mathbf{C}_z= d_{\mathbf{M}_z}, \mathbf{C}_y = d_{\mathbf{M}_y},
\end{align}
where $d_{\mathbf{M}}$ denotes the Mahalanobis distance: $d^2_{\mathbf{M}}(\x,\y):=(\x-\y)^T{\mathbf{M}}(\x-\y)$ and ${\mathbf{M}}$ is some positive semidefinite matrix.  
The Mahalanobis distance generalizes the squared Euclidean distance and allows us to consider different costs based on correlations between features.
The framework of Mahalanobis distance benefits from efficient metric learning techniques \cite{cuturi2014ground}; recent research also establishes connections between it and robust \ot \cite{paty2019subspace,dhouib2020swiss}. When a simple L2-distance is used ($\mathbf{M}=\mathbf{I}$), we will denote this specific variant as \lot\hspace{-1mm}.

When \alg moves source points through anchors, the anchors impose a type of bottleneck, and this results in a loss of information that makes it difficult to estimate the corresponding point in the target space.
In cases where accurate point-to-point alignment is desired,
we propose an alternative strategy for defining the cost matrix $\mathbf{C}_z$. The idea is to represent an anchor as the distribution of points assigned to it. Specifically, we represent $\mathbf{z}^x,\mathbf{z}^y$ as measures in $\mathbb{R}^d$: $\tilde{\mathbf{z}}^x=\sum_{i=1}^N\mathbf{P}_x(\mathbf{x}_i|\mathbf{z}^x)\delta_{\mathbf{x}_i},$ $\tilde{\mathbf{z}}^y=\sum_{j=1}^M\mathbf{P}_y(\mathbf{y}_j|\mathbf{z}^y)\delta_{\mathbf{y}_j}$. Then we measure the cost between anchors as the squared Wasserstein distance between their respective distributions, 
\begin{equation}
\label{eq:newcost}
\mathbf{C}_z:=[{\cal{W}}^2_2(\mathbf{P}_x(\cdot|\mathbf{z}^x_m),\mathbf{P}_y(\cdot|\mathbf{z}^y_n))]_{m,n}.
\end{equation}
Besides the quantity itself, the transport plan returned by calculating $\mathbf{C}_z$ is also very important as it provides accurate point-to-point maps. Since the cost matrix is now a function of  $\mathbf{P}_x$ and $\mathbf{P}_y$, we use an additional alternating scheme to solve the problem: 
we alternate between updating $\mathbf{C}_z$ while keeping $\mathbf{P}_x$ and $\mathbf{P}_y$ fixed, and then updating $\mathbf{P}_x,\mathbf{P}_y,\mathbf{P}_z$ while keeping $\mathbf{C}_z$ fixed.
An efficient algorithm is presented in Appendix \ref{wasslot} to reduce the computation complexity. This variant, \lotwa, can yield better performance in downstream tasks that require precise alignment at the cost of additional computation.

\subsection{Algorithm}

\label{sec:algorithm}
In the rest of this section, we will develop our main approach for solving the
problem in Def.~\ref{def1}. We provide an outline of the algorithm in Algorithm~\ref{alg:1} and an implementation of the algorithm in Python at:~ \url{http://nerdslab.github.io/latentOT}.


\paragraph{\em (1) Optimizing $\mathbf{P}_x,\mathbf{P}_y$ and $\mathbf{P}_z$:}
To begin, we assume that the anchors and cost matrices $\mathbf{C}_x,\mathbf{C}_z,\mathbf{C}_y$ are already specified. 
Let $\mathbf{K}_x,\mathbf{K}_z,\mathbf{K}_y$ be the Gibbs kernels induced from the cost matrices $\mathbf{C}_x,\mathbf{C}_z,\mathbf{C}_y$ as in (\ref{regot}). The optimization problem can be written as,
\begin{align}\label{cost1for}
    \min_{\mathbf{u}_z,\mathbf{v}_z,\mathbf{P}_x,\mathbf{P}_z,\mathbf{P}_y} & \sum _{i\in \{x,y,z\}}\varepsilon_i\text{KL}(\mathbf{P}_i\|\mathbf{K}_i),
    \nonumber\\ \text{subject to:~ }&\mathbf{P}_x\mathbf{1} =\mu,\mathbf{P}^T_x\mathbf{1}=\mathbf{u}_z,\mathbf{P}_z\mathbf{1} = \mathbf{u}_z,\nonumber
    \\&\mathbf{P}_z^T\mathbf{1} =\mathbf{v}_z,\mathbf{P}_y\mathbf{1}=\mathbf{v}_z,\mathbf{P}^T_y\mathbf{1}=\nu. 
\end{align}
This is a Bregman projection problem with affine constraints. An iterative projection procedure can thus be applied to solve the problem \cite{benamou2015iterative}. 
We present the procedure as \textsc{{\color{blue}UpdatePlan}} in Algorithm \ref{alg:1}, where $\mathbf{P}_x,\mathbf{P}_z,\mathbf{P}_y$ are successively projected onto the constrained sets of fixed marginal distributions. We defer the detailed derivation to Appendix \ref{der1}.

\begin{algorithm}[t]
\begin{algorithmic}[1]
\caption{Latent Optimal Transport (LOT)} \label{alg:1}
\addtocounter{algorithm}{-1}
\INPUT{Data matrices  $\mathbf{X}$, $\mathbf{Y}$; metric costs ${\bf M}_x, {\bf M}_y, {\bf M}_z$; entropy regularization parameters $\varepsilon_x$, $\varepsilon_y$, $\varepsilon_z$; initial anchors $\mathbf{Z}_x,\mathbf{Z}_y$.}
\OUTPUT{Transport plans $\mathbf{P}_x,\mathbf{P}_y,\mathbf{P}_z$; source and target anchors $\mathbf{Z}_x$, $\mathbf{Z}_y$.}
\ENSURE{$\mathbf{P}_x,\mathbf{P}_y,\mathbf{P}_z,\mathbf{Z}_x,\mathbf{Z}_y$}
\WHILE{not converging}
\STATE{$(\text{vec}(\mathbf{Z}_x),\text{vec}(\mathbf{Z}_y)) \leftarrow$ Eqn.~(\ref{cenfor0})}
\vspace{0.5mm}
\STATE{$\mathbf{K}_x = \{\exp(-\|\mathbf{X}[i]-\mathbf{Z}_x[j]\|_{\mathbf{M}_x}^2/\varepsilon_x)\}_{i,j}$} 
\STATE{$\mathbf{K}_y = \{\exp(-\|\mathbf{Y}[i]-\mathbf{Z}_y[j]\|_{\mathbf{M}_y}^2/\varepsilon_y)\}_{j,q}$}
\STATE{$\mathbf{K}_z = \{\exp(-\|\mathbf{Z}_x[i]-\mathbf{Z}_y[j]\|_{\mathbf{M}_z}^2/\varepsilon_z)\}_{p,q}$}
\vspace{0.5mm}
\STATE{$\mathbf{P}_x, \mathbf{P}_y, \mathbf{P}_z\leftarrow${{\color{blue}\textsc{UpdatePlan}}}~($\mathbf{K}_x,\mathbf{K}_y,\mathbf{K}_z$})
\ENDWHILE
\REQUIRE{${\bf P}_x$, ${\bf P}_y$, ${\bf P}_z$, ${\bf Z}_x$, ${\bf Z}_y$}
\end{algorithmic}
\vspace{-2mm}
\hrulefill

{{\textsc{\color{blue}UpdatePlan}}}~({$\mathbf{K}_x,\mathbf{K}_y,\mathbf{K}_z$})

\begin{algorithmic}[1]
\caption{\hspace{2mm} Latent Optimal Transport - \alg} \label{cc}

\ENSURE{$\mathbf{\alpha}_x\leftarrow \mathbf{1}_N;\mathbf{\beta}_x\leftarrow \mathbf{1}_{k_1};\mathbf{\alpha}_y\leftarrow \mathbf{1}_{k_2}$}; 

{~~~~~$\mathbf{\beta}_y\leftarrow \mathbf{1}_M;\mathbf{\alpha}_z\leftarrow \mathbf{1}_{k_1};\mathbf{\beta}_z\leftarrow \mathbf{1}_{k_2}$ }
\vspace{0.5mm}
\WHILE{not converging}
\STATE $\alpha_x\leftarrow \mu \oslash \mathbf{K}_x \beta_x$; $\beta_y\leftarrow \nu \oslash \mathbf{K}_y^T \alpha_y$ 
\vspace{0.5mm}
\STATE $\mathbf{u}_z \leftarrow ((\alpha_z\odot \mathbf{K}_z\beta_z)\odot(\beta_x\odot \mathbf{K}_x^T\alpha_x))^{\frac{1}{2}}$
\vspace{0.5mm}
\STATE $\beta_x\leftarrow \mathbf{u}_z\oslash \mathbf{K}_x^T\alpha_x$; $\alpha_z\leftarrow \mathbf{u}_z\oslash \mathbf{K}_z\beta_z$ 
\vspace{0.5mm}
\STATE $\mathbf{v}_z \leftarrow ((\alpha_y\odot \mathbf{K}_y\beta_y)\odot(\beta_z\odot \mathbf{K}_z^T\alpha_z))^{\frac{1}{2}}$
\vspace{1mm}
\STATE $\beta_z\leftarrow \mathbf{v}_z\oslash \mathbf{K}_z^T\alpha_z$; $\alpha_y\leftarrow \mathbf{v}_z\oslash \mathbf{K}_y\beta_y$ 
\ENDWHILE
\vspace{0.5mm}
 
 \REQUIRE{$ \mathbf{P}_i = {\text{diag}}(\alpha_i)\mathbf{K}_i{\text{diag}}(\beta_i)$, $i\in\{x,y,z\}$}
\end{algorithmic}
\end{algorithm}

\vspace{-2mm}
\paragraph{\em (2) Optimizing the anchor locations:} Now we consider the case where we are free to select the anchor locations in $\mathbb{R}^d$.
We consider the class of Mahalanobis costs described in Section~\ref{sec:costs}. Let ${\mathbf{M}}_x$, ${\mathbf{M}}_z$, ${\mathbf{M}}_y$ be the  Mahalanobis matrices correspond to $\mathbf{C}_x$, $\mathbf{C}_z$, and $\mathbf{C}_y$, respectively.

Given the transport plans generated after solving (\ref{cost1for}), we can derive the the first-order stationary condition of OT$^L$ with respect to $\mathbf{Z}_x$ and $\mathbf{Z}_y$. Let 
\begin{align}
{\bf A} =\nonumber \begin{bmatrix}
D(\mathbf{u}_z)\otimes (\mathbf{M}_x+\mathbf{M}_z) & \mathbf{P}_z\otimes \mathbf{M}_z\\
-\mathbf{P}_z^T\otimes \mathbf{M}_z & D(\mathbf{v}_z)\otimes (\mathbf{M}_y+\mathbf{M}_z) 
\end{bmatrix}    
\end{align} 
The update formula is given by 
\begin{align}\label{cenfor0}
\begin{bmatrix}\text{vec}(\mathbf{Z}_x^*)\\\text{vec}(\textbf{Z}_y^*)
    \end{bmatrix}
    ={\bf A}^{-1} \times \begin{bmatrix}(\mathbf{P}_x^T\otimes \mathbf{M}_x)\text{vec}(\mathbf{X})\\(\mathbf{P}_y\otimes \mathbf{M}_y)\text{vec}(\mathbf{Y})
    \end{bmatrix},
\end{align}
where $\text{vec}(\cdot)$ denotes the operator converting a matrix to a column vector, and $D(\cdot)$ denotes the operator converting a vector to a diagonal matrix. We defer the detailed derivation to Appendix \ref{b:anchoropt}. Pseudo-code for the combined scheme can be found in Algorithm~\ref{alg:1}. 

\vspace{-2mm}
\paragraph{\em (3) Robust estimation of data transport:}
\alg provides robust transport 
in the target domain by aligning the data through anchors, which can facilitate regression, and classification in downstream applications. We denote the centroids of the source and target by $\mathbf{Q}_x = \text{diag}(\mathbf{u}_z^{-1})\mathbf{P}_x^T\mathbf{X}^T$, $\mathbf{Q}_y = \text{diag}(\mathbf{v}_z^{-1})\mathbf{P}_y\mathbf{Y}^T$. We propose the estimator $\hat{\mathbf{X}}:=\sum_{m,n} p(\z_m^x,\z_n^y|\x)(\mathbf{Q}^y_m-\mathbf{Q}^x_n)= \text{diag}(\mu^{-1})\mathbf{P}_x\text{diag}((\mathbf{P}_z\mathbf{1})^{-1})\mathbf{P}_z(\mathbf{Q}_y-\mathbf{Q}_x).$
In contrast to factored coupling \cite{forrow2019statistical}, where $\mathbf{Z}_x=\mathbf{Z}_y$, \alg is robust even when the source and target have different structures (see Table~\ref{table:table1} MNIST-DU, Figure~\ref{fig:mnist}). 

\vspace{-2mm}
\paragraph{\em (4) Implementation details:}
\alg has two primary hyperparameters that must be specified: (i) the number of the source and target anchors $k_x,k_y$ and (ii) the regularization parameter $\varepsilon$. For details on the tuning of these parameters, please refer to Appendix \ref{sec:hyptun}. In practice, we use centroids from k-means clustering \cite{arthur2006k} to initialize the anchors, and for all the experiments we have conducted, \alg typically converges within $20$ iterations.

\section{Related Work}\label{sec:related work}

{\bf Interpolation between factored coupling and k-means clustering:}~
Assume we select the Mahalanobis matrices of the costs defined in Section \ref{sec:costs} to be $\mathbf{M}_x=\mathbf{M}_y=\mathbf{I},$ and $\mathbf{M}_z=\lambda\mathbf{I}$. If we let $\lambda\to \infty$ when estimating the transport between source and target anchors, the anchors merge, and our approach reduces to the case of factored coupling \cite{forrow2019statistical}. At the other end, if we let $\lambda\to 0$, then \alg becomes separable, and the middle term vanishes. In this case, each remaining term exactly corresponds to a pure clustering task, and \alg reduces to k-means clustering \cite{arthur2006k}. 

{\bf Relationship to \ot\hspace{-1mm}-based clustering methods:}~
Many methods that combine \ot and clustering \cite{li2008real,ye2017fast,ho2017multilevel,dessein2017parameter,genevay2019differentiable,alvarez2020geometric} focus on using the Wasserstein distance to identify barycenters that serve as the centroids of clusters. When finding barycenters for the source and target separately, this could be seen as \alg with $\mathbf{C}_z= 0$ and $\mathbf{C}_x$, $\mathbf{C}_y$ defined using a squared L2 distance. In other related work \cite{laclau2017co}, co-clustering is applied to a transport plan as a post-processing operation, and no additional regularization on the transportation cost in \ot is imposed. In contrast, our approach induces explicit regularization by separately defining cost matrices for the transport between the source/target points and their anchors. This yields a transport plan guided by a cluster-level matching. 

{\bf Relationship to hierarchical \ot:}~ Hierarchical \ot \cite{chen2018optimal,lee2019hierarchical,yurochkin2019hierarchical,xu2020learning} transports points by moving them within some predetermined subgroup simultaneously based on either their class label or pre-specified structures, and then forms a matching of these subgroups using the Wasserstein distance. The resulting problem solves a multi-layer \ot problem which gives rise to its name. With a Wasserstein distance used to build the $\mathbf{C}_z$ cost matrix, \alg effectively reduces to hierarchical \ot for fixed and hard-class assignment $\mathbf{P}_x$ and $\mathbf{P}_y$. However, a crucial difference between \alg and hierarchical \ot lies in that the latter imposes the \textit{known} structure information. In contrast, \alg discovers this structure by simultaneously learning $\mathbf{P}_x$ and $\mathbf{P}_y$. 

{\bf Transportation with anchors:}~ The notion of moving data points with anchors to match points in heterogeneous spaces has appeared in other work \cite{sato2020fast,manay2006integral}. These approaches map each point from one domain into a distribution of the costs, which effectively builds up a common representation for the points from both spaces. In contrast to this work, 
we use the anchors to encourage clustering of data and to impose rank constraints on the transport plan. 

\section{Theoretical Analysis}
\label{sec:analysis}

{\bf LOT as a relaxation of OT:}~
We now ask how the optimal value of our original rank-constrained objective in (\ref{cost1for}) is related to the transportation cost defined in entropy-regularized \url{OT}.
It turns out their objectives are connected by an inequality described below (see Appendix \ref{sec:suppaproof} for a proof).

\begin{proposition}\label{prop:relax}
 Let $\bf P$ be a transport plan of the form in (\ref{form}). Assume that $\mathbf{K}$ is some Gibbs kernel that satisfies, \begin{align}\label{kercond}
 \mathbf{K}_x\mathbf{K}_z\mathbf{K}_y\leq \mathbf{K},
 \end{align} where the inequality is over each entry. Then we have, 
 \vspace{-2mm}
\begin{align}
    \varepsilon \text{KL}(\mathbf{P}\|\mathbf{K}) \leq &\varepsilon (\text{KL}(\mathbf{P}_x\|\mathbf{K}_x) +\text{KL}(\mathbf{P}_z\|\mathbf{K}_z) \nonumber\\
    + &\text{KL}(\mathbf{P}_y\|\mathbf{K}_y))+\varepsilon(\mathbf{H}(\mathbf{u}_z)+ \mathbf{H}(\mathbf{v}_z)),
\end{align}
where $\mathbf{H}(\mathbf{a}):=-\sum_i \mathbf{a}_i\log \mathbf{a}_i$ denotes the
entropy.
\end{proposition}
The proposition shows that an OT objective, corresponding to a kernel $\mathbf{K}$ (resp. $\mathbf{C}$), can be upper bounded by three sub-OT problems defined by subsequent kernels $\mathbf{K}_x, \mathbf{K}_z,\mathbf{K}_y$ (resp. $\mathbf{C}_x,\mathbf{C}_z,\mathbf{C}_y$) that satisfies (\ref{kercond}) (resp. $\exp(-\mathbf{C}_x/\varepsilon)\exp(-\mathbf{C}_z/\varepsilon)\exp(-\mathbf{C}_y/\varepsilon)\leq \exp(-\mathbf{C}/\varepsilon)$). 

Let us compare the upper bound given by Proposition \ref{prop:relax} with Def. \ref{def1} and ignore the entropy terms; we recognize that it is precisely the entropy-regularized objective of \alg.   
In other words, with suitable cost matrices satisfying (\ref{kercond}), \alg could be interpreted as a relaxation of an OT problem in a decomposed form. We then ask what $\mathbf{C}_x$, $\mathbf{C}_z$, $\mathbf{C}_y$ should be to satisfy (\ref{kercond}). In cases where cost $\mathbf{C}$ is defined by the $L^p$-norm to the power $p$, the following corollary shows that the same form suffices.
\begin{corollary}\label{cor:relate}
Let $d(\x,\y):=\|\x-\y\|_p^p$. Consider an optimal transport problem {\rm OT}$_{\mathbf{C},\varepsilon}$ with cost $\mathbf{C}[i,j]= d(\x_i,\y_j)$, where $p\geq 1$. 
Then for a sufficiently small $\varepsilon$, the latent optimal transport OT$^L$ with cost matrices, $\mathbf{C}_x[i,m]=3^{p-1}d(\x_i-\z^x_m),\mathbf{C}_z[m,n]=3^{p-1}d(\z^x_m-\z^y_n),\mathbf{C}_y[n,j]=3^{p-1}d(\z^y_n- \y_j)$ minimizes an upper bound of the entropy-regularized \ot objective in (\ref{regot}).
\end{corollary}
Corollary (\ref{cor:relate}) provides natural costs for \alg to be posed as a relaxation to a \ot problem with $L^p$ norm. More generally, finding the optimal cost functions that obey (\ref{kercond}) and minimize the gap in the inequality in Proposition \ref{prop:relax} is outside the scope of this work but would be an interesting topic for future investigation.

\begin{figure*}[t!]
\centering
\subfloat{\includegraphics[width=.5\textwidth]{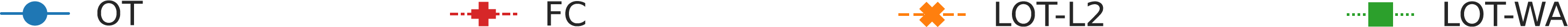}}
    \setcounter{subfigure}{0}
    \subfloat[][\footnotesize{Rotation angle}]{\includegraphics[width=0.2\textwidth]{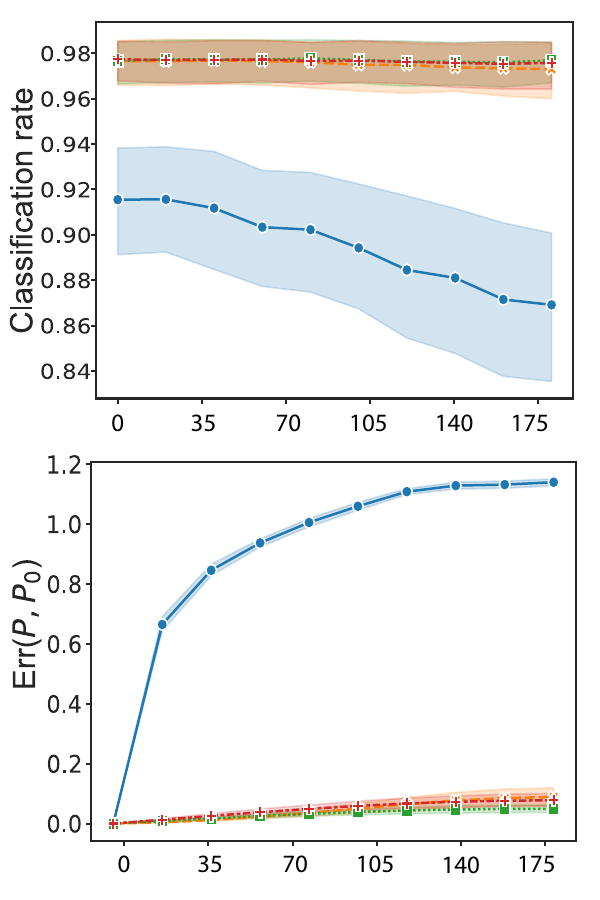}\vspace{-3mm}}
      \subfloat[][\footnotesize{Outlier rate}]{\includegraphics[width=0.2\textwidth]
   {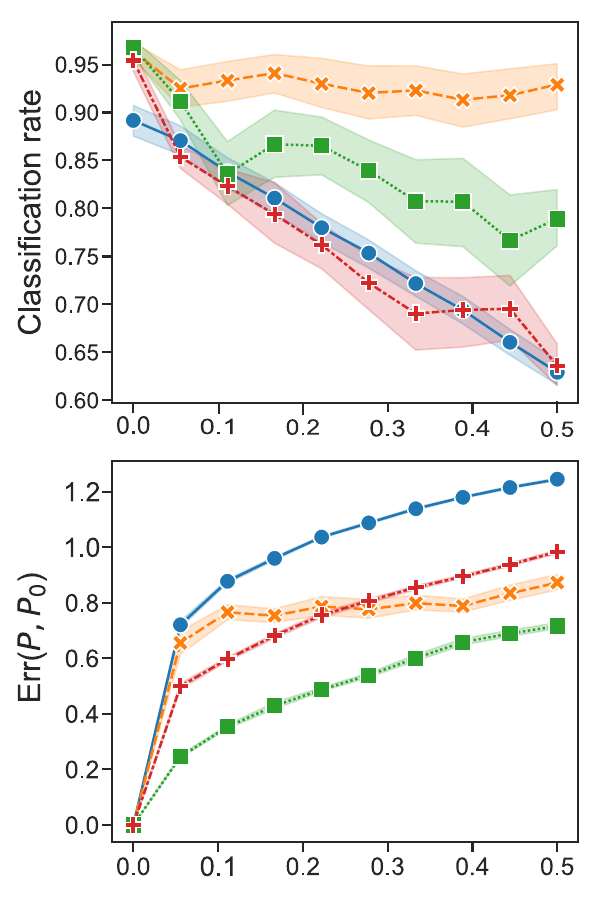}\vspace{-3mm}}
   \subfloat[][\footnotesize{Dimensionality}]{\includegraphics[width=0.2\textwidth]
   {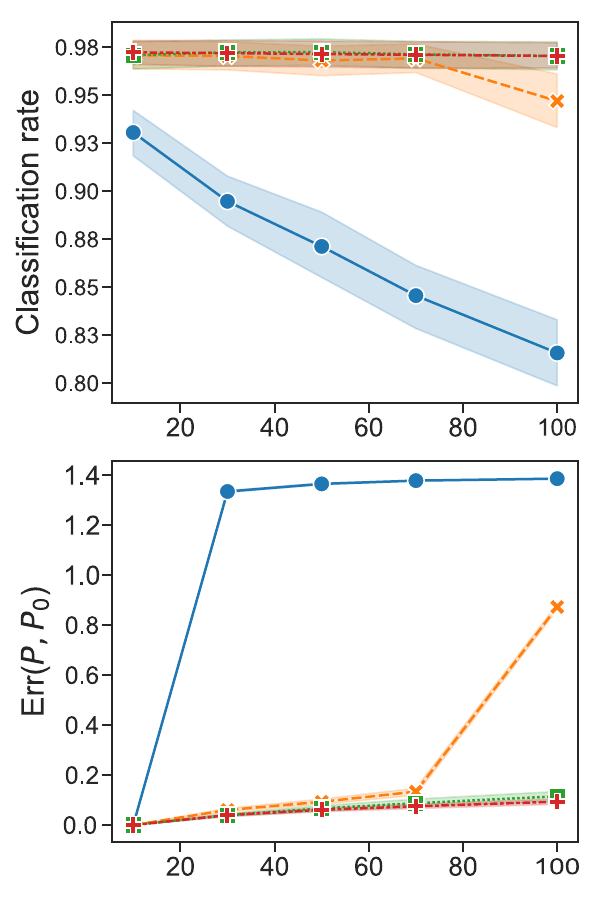}\vspace{-3mm}}
    \subfloat[][\footnotesize{Cluster mismatch}]{\includegraphics[width=0.2\textwidth]
   {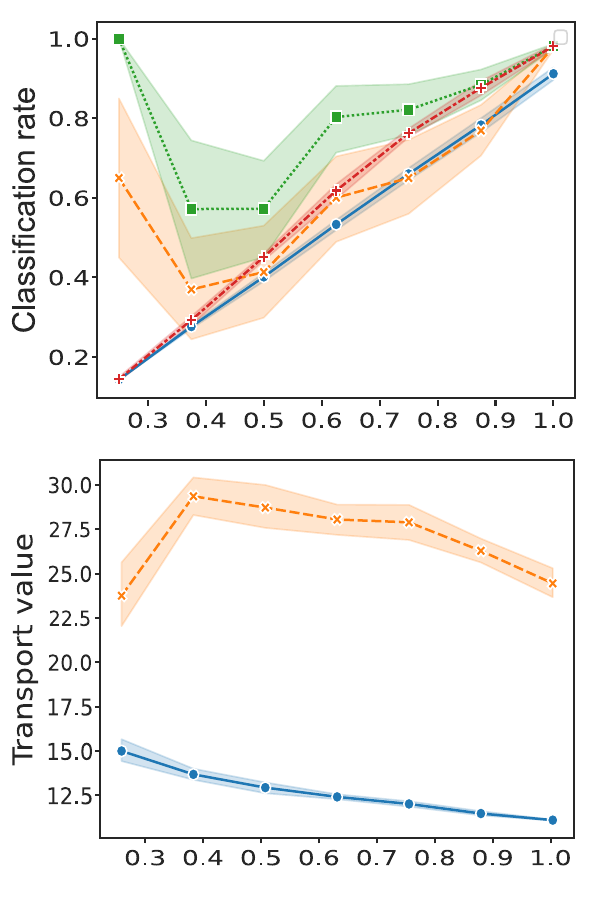}\vspace{-3mm}}
     \subfloat[][\footnotesize{Transport rank}]{\includegraphics[width=0.2\textwidth]
   {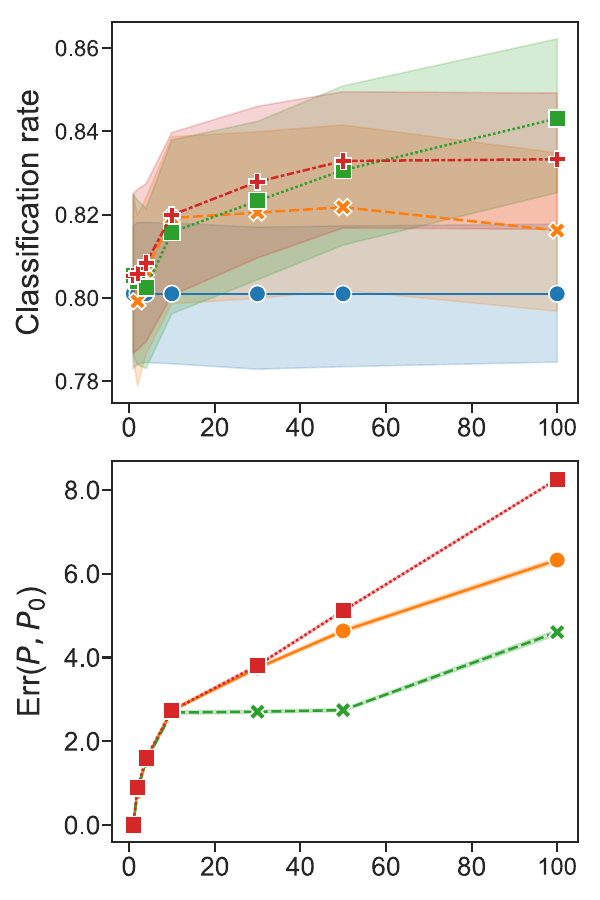}\vspace{-3mm}}
\caption{\footnotesize{{\bf \em Results on Gaussian mixture models.} In (a), we apply a rotation between the source and target, in (b) we add outliers, in (c) we vary the ambient dimension, in (d) the target is set to have $8$ components, and we vary the number of components in the source to simulate source-target mismatch, in (e) we fix the rank to 10 and vary the number of factors (anchors) used in the approximation. 
Throughout, we simulate data according to a GMM and evaluate performance by measuring the classification accuracy (top) and computing the deviation between the transport plans before and after the perturbations with respect to the Fr\"obenius norm (bottom).} \label{fig:gmm}}
\end{figure*}

{\bf Sampling complexity:}~
Below we analyze \alg from a statistical point of view. Specifically, we bound the sampling rate of $\rm{OT}^L$ in Def. \ref{def1} when the true distributions $\mu$ and $\nu$ are estimated by their empirical distributions.
\begin{proposition}
\label{prop:sampling}
Suppose $X$ and $Y$ have distributions $\mu$ and $\nu$ supported on a compact region $\Omega$ in $\mathbb{R}^d$, the cost functions $c_x(\cdot,\cdot)$ and $c_y(\cdot,\cdot)$ are defined as the squared Euclidean distance, and $\hat{\mu}$, $\hat{\nu}$ are empirical distributions of $n$ and $m$ i.i.d. samples from $\mu$ and $\nu$, respectively. If the spaces for latent distributions are equal to ${\cal{Z}}_x={\cal{Z}}_y=\mathbb{R}^d$, and there are $k_x$ and $k_y$ anchors in the source and target, respectively, then with probability at least $1-\delta$,  
\begin{align}
    {\rm Err}  ~ \leq~ &C\sqrt{\frac{k_{max}^3d \log k_{max}+ \log (2/\delta)}{N}},
\end{align}
where ${\rm Err} = |{\rm OT}^L(\mu,\nu) - {\rm OT}^L(\hat{\mu},\hat{\nu})|$, $k_{max}=\max\{k_x,k_y\}$, $N= \min\{n,m\}$ and $C\geq 0$ is some constant not depending on $N$.
\end{proposition}
\vspace{-3mm}
As shown in \cite{weed2019sharp}, the general sampling rate of a plug-in \ot scales with $N^{\frac{1}{d}}$, suffering from the ``curse of dimensionality''. On the other hand, as evidence from \cite{forrow2019statistical}, structural optimal transport paves ways to overcome the issue. In particular, \alg achieves $N^{-\frac{1}{2}}$ scaling by regularizing the transport rank.

{\bf Time complexity:} 
We can bound the time complexity as $O( T_{\rm i} + T_{\rm{bcd}}(T_{\rm k} + T_{\rm au} + T_{\rm pu}))$, where $T_{\rm i}$ is the initialization complexity, e.g., if we use k-means, then it equals to $O(nk_xdT_x + mk_ydT_y)$ where $T_x$ and $T_y$ are the iteration numbers of the Floyd algorithm applied to the source and target, respectively, 
$T_{\rm bcd}$ is the total number of iterations of block-coordinate descent, $T_{\rm k} = O(nk_x + mk_y)$ is the computation time for updating the kernels,  $T_{\rm au} = O((k_x + k_y)^3 + d(nk_x + mk_y))$ is the complexity of updating anchors, and $T_{\rm pu}$ is the complexity for updating plans. Because our updates are based on  iterative Bregman projections similar to the Sinkhorn algorithm, it has complexity comparable to OT. Therefore, the overall complexity of \alg is approximately $T_{\rm bcd}$ times of OT, assuming $n,~m\geq d(k_x+k_y)$. Empirically, $T_{\rm bcd}$ depends on the structure of data, but we observed that it is usually under $20$. Note that the same applies to \url{FC} with $k_x=k_y=k$. 
In Figure \ref{fig:time}, we complement our analysis by simulating a comparison of 
the time complexity for \url{LOT} and \url{FC} vs. OT in the setting of a 7-component Gaussian mixture model.
We can see the compute time of \url{LOT} scales similarly to \url{FC}. 

\paragraph{Transport cost:} We also compare the transport loss returned by \url{LOT} (blue), \url{FC} (orange), and OT (green) as a function of the number of anchors in Figure \ref{fig:time}. For a fair comparison, we considered a balanced scenario where 7-component GMMs generate the source and target. The anchors of the source and target are chosen to be equal for \alg. The result shows that the losses are indeed higher for \url{LOT} and \url{FC} compared to OT but are fairly insensitive with to the chosen number of anchors. Moreover, we find that \alg has a slightly lower loss compared to \url{FC} even when we choose the number of source and target anchors to be equal.

\begin{figure}[t!]
\centering
    \setcounter{subfigure}{0}
   \subfloat{\includegraphics[height=.26\textwidth]
   {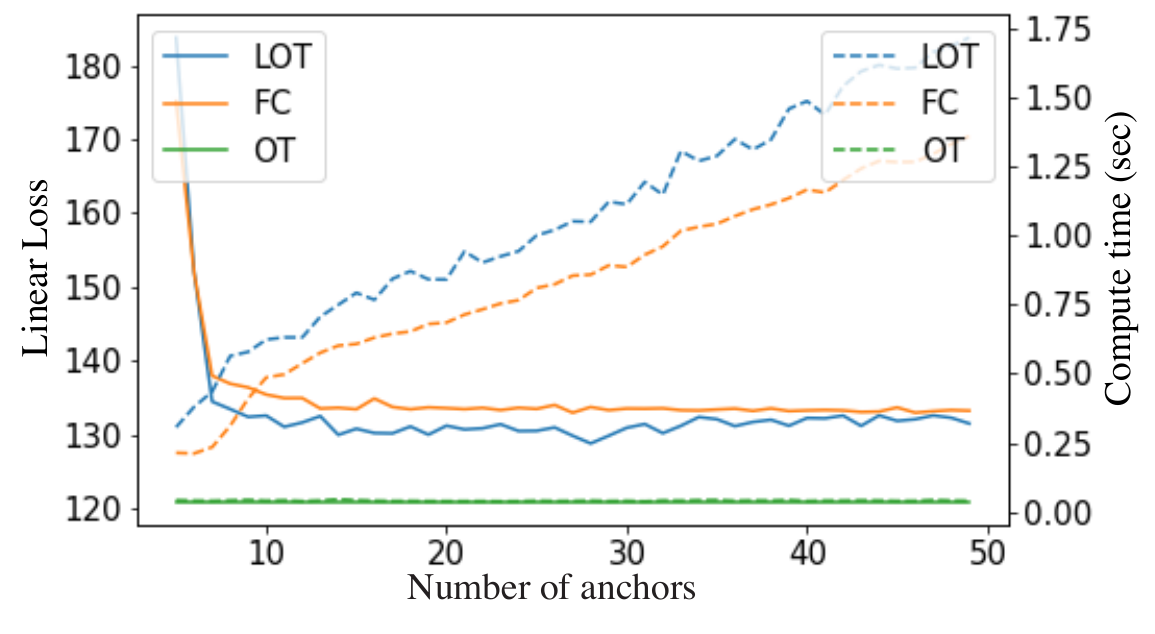}}
    \caption{\footnotesize {\bf{\em Comparisons of the time complexity and loss.}} The figure compares the time complexity (dashed) and linear loss (solid) of \url{LOT}, \url{FC}, and OT in the setting of the 7-component GMM model. \label{fig:time}\vspace{-5mm}}
\end{figure}

\section{Experiments}\label{sec:experiment}
In this section, we conduct empirical investigations. Details of hyperparameter tuning can be found in Appendix \ref{sec:hyptun}.
\label{sec:exp}

{\bf E1) Testing robustness to various data perturbations:}~To better understand how different types of domain shift impact the transport plans generated by our approach, we considered different transformations between the source and target. To create synthetic data for this task, we generated multiple clusters/components using a k-dimensional Gaussian with random mean and covariance sampled from a Wishart distribution, randomly projected to a 5-dimensional subspace. The source and target are generated independently: we randomly sample a fixed number of points according to the true distribution for each cluster. We compared the performance of the \alg variants proposed in Section \ref{sec:costs}: \lot (orange curves) and \lotwa (green curves) with baselines OT  (blue curves) and rank regularized factored coupling (\fc\hspace{-0.95mm}) \cite{graf2007foundations} (red curves) in terms of their (i) classification rates and (ii) deviation from the original transport plan without perturbations, which we compute as $ \text{Err}({\bf P} - {\bf P}_0) = \|{\bf P} - {\bf P}_0 \|_F / \| {\bf P}_0 \|_F$, where ${\bf P}_0$ is the transport plan obtained before perturbations. The results are averaged over 20 runs, and a 75\% confidence interval is used. See Appendix \ref{sec:detailexp2} for further details.

When compared with OT, both our method and \fc provide more stable class recovery, even with significant amounts of perturbations (Figure~\ref{fig:gmm}).  When we examine the error term in the transport plan, we observe that, in most cases, the OT plan deviates rapidly, even for small amounts of perturbations. Both \fc and \alg appear to have similar performances across rotations while OT's performance decreases quickly. In experiment (b), we found that both \alg variants provide substantial improvements on classification subject to outliers, implying the applicability of \alg for noisy data. In experiment (c), we study \alg in the high-dimensional setting; we find that \lotwa behaves similarly to \fc with some degradation in performance after the dimension increases beyond $70$. Next, in experiment (d), we fix the number of components in the target to be $10$, while varying the number in the source from 4 to 10. In contrast to the outlier experiment in (b),  \lotwa shows more resilience to mismatches between the source and target. At the bottom of plot (d), we show the 2-Wasserstein distance (blue) and latent Wasserstein discrepancy (orange) defined in Proposition \ref{prop1}. This shows that the latent Wasserstein discrepancy does indeed provide an upper bound on the 2-Wasserstein distance. Finally, we look at the effect of transport rank on \alg and \fc in (e). The plot shows that the slope for \alg is flatter than \fc while maintaining similar performances. 

{\bf E2) Domain adaptation application:}~In our next experiment, we used \alg to correct for domain shift in a neural network that is trained on one dataset but underperforms on a new but similar dataset (Table~\ref{table:table1}, Figure~\ref{fig:mnist}).
MNIST and USPS are two handwritten digits datasets that are semantically similar but that have different pixel-level distributions and thus introduce domain shift (Figure~\ref{fig:mnist}a).
We train a multi-layer perceptron (MLP) on the training set of the MNIST dataset, freeze the network, and use it for the remaining experiments. The classifier achieves 100\% training accuracy and a 98\% validation accuracy on MNIST but only achieves  79.3\% accuracy on the USPS validation set. We project MNIST's training samples in the classifier's output space (logits) and consider the 10D projection to be the target distribution. 
Similarly, we project images from the USPS dataset in the network's output space to get our source distribution.
We study the performance of \alg in correcting the classifier's outputs and compare with \fc, k-means OT (\kot\hspace{-0.95mm}) \cite{forrow2019statistical}, and subspace alignment (SA) \cite{fernando2013unsupervised}.

In Table~\ref{table:table1}, we summarize the results of our comparisons on the domain adaptation task (MNIST-USPS). Our results suggest that both \fc and \alg perform pretty well on this task, with \alg beating \fc by 2\% in terms of their final classification accuracy. We also show that \alg does better than naive \kot. In Figure~\ref{fig:mnist}a, we use Isomap to project the distribution of USPS images as well as the alignment results for \alg\hspace{-1mm}, \fc, and OT. For both \alg and \fc, we also display the anchors; note that for \alg, we have two different sets of anchors (source, red; target, blue). This example highlights the alignment of the anchors in our approach and contrasts it with that of \fc. 

\begin{table}[t!]
\caption{{\footnotesize {\bf \em{Results for concept drift and domain adaptation for handwritten digits.}}  The classification accuracy and L2-error are computed after transport for MNIST to USPS (left) and coarse dropout (right). Our method is compared with the accuracy before alignment (Original), entropy-regularized OT, k-means plus OT (\kot\hspace{-1mm} ), and subspace alignment (SA). 
\label{table:table1}}}
\begin{center}
\begin{tabular}{l|c|cc}
         & MNIST-USPS  & \multicolumn{2}{c}{MNIST-DU}         \\
         & Accuracy        & Accuracy      & L2 error               \\ \hline
Original & 79.3            & 72.6          & 0.72                   \\
OT       & 76.9            & 61.5          & 0.71                   \\
\kot      & 79.4            & 60.9          & 0.73                   \\
\url{SA}       & 81.3            & 72.3          & -                      \\
\fc       & 84.1            & 67.2          & 0.59                   \\
\lotwa     & \textbf{86.2}   & \textbf{77.7} & \textbf{0.56}
\end{tabular}
\end{center}
\vspace{-5mm}
\end{table} 

\begin{figure}[t!]

    \begin{subfigure}[t]{\textwidth}
    \subfloat{\includegraphics[height=0.966\textwidth]
   {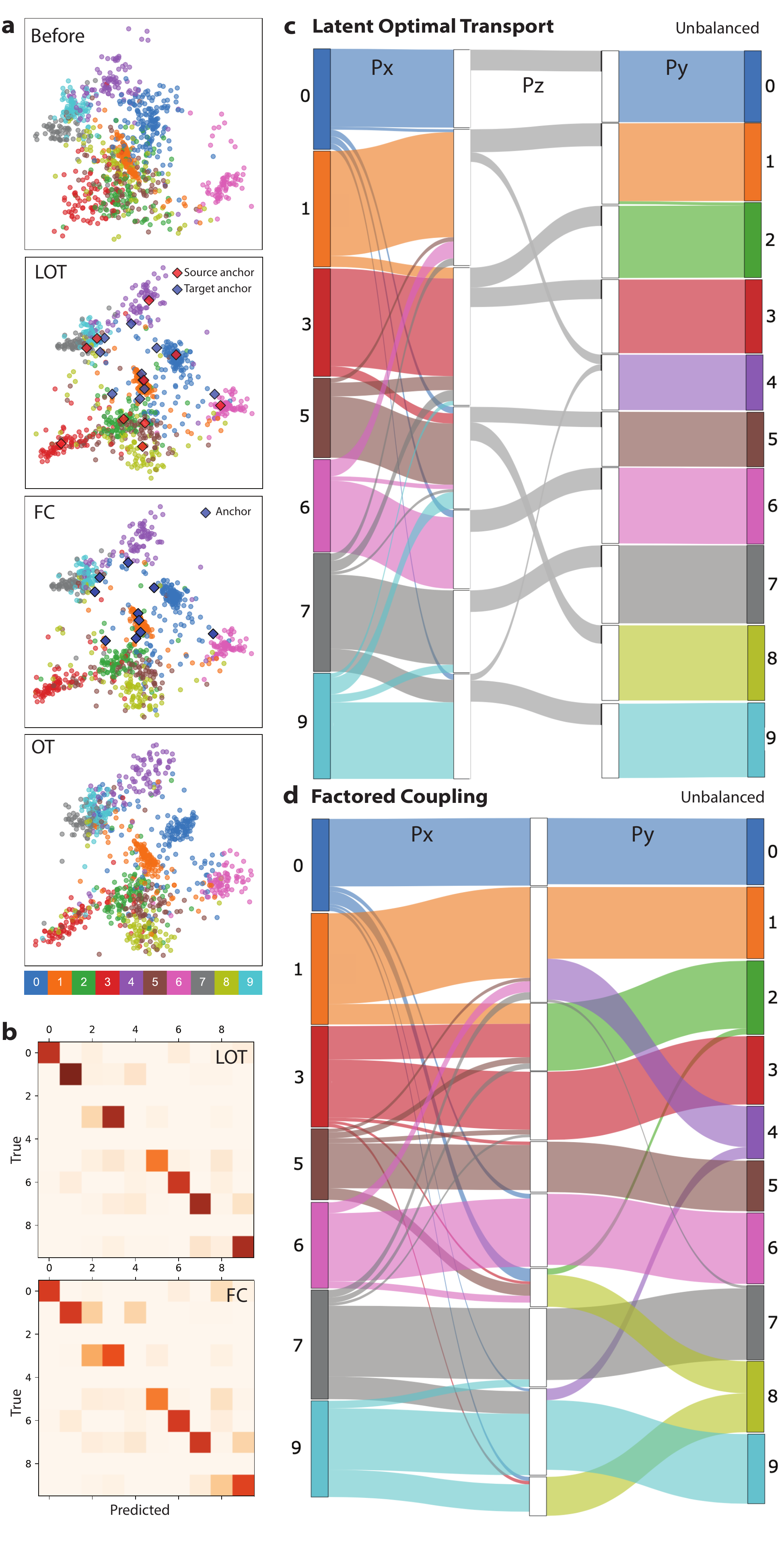}}
   \end{subfigure}
   
    \caption{\footnotesize {{\bf \em Visualization of results on handwritten digits and examples of domain shift.} (a) 2D projections of representations formed in deep neural network before (top) and after different alignment methods (\alg, \fc, OT). (b) confusion matrices for \alg (top) and \fc (bottom) after alignment. The transport plans are visualized for \alg (c) and \fc (d) in the unbalanced case.} \label{fig:mnist}}
\end{figure}

Taking inspiration from studies in self-supervised learning \cite{doersch2015unsupervised,he2020momentum} that use different transformations of an input image (e.g., masking parts of the image) to build invariances into a network's representations, here we 
ask how augmentations of the images introduce domain shift and whether our approach can correct/identify it. 
To test this, we apply coarse dropout on test samples in MNIST and feed them to the classifier to get a new source distribution. We do this in a balanced (all digits in source and target) and an unbalanced setting (2, 4, 8 removed from source, all digits in target). The results of the unbalanced dropout are summarized in Table~\ref{table:table1} (MNIST-DU), and the other results are provided in Table~\ref{table:table1_extended} in the Appendix. In this case, we have the features of the testing samples pre-transformation, and thus, we can compare the transported features to the ground truth features in terms of their point-to-point error (L2 distance). In the unbalanced case, we observe even more significant gaps between \fc and \alg\hspace{-1mm}, as the source and target datasets have different structures. To quantify these different class-level errors, we compare the confusion matrices for the classifier's output after alignment (Figure~\ref{fig:mnist}b). By examining the columns corresponding to the removed digits, we see that \fc is more likely to misclassify these images. Our results suggest that \alg has comparable performance with \fc in a balanced setting and outperforms \fc in an unbalanced case.

\begin{figure}[t!]
\centering
    \setcounter{subfigure}{0}
   \subfloat{\includegraphics[width=0.49\textwidth]
   {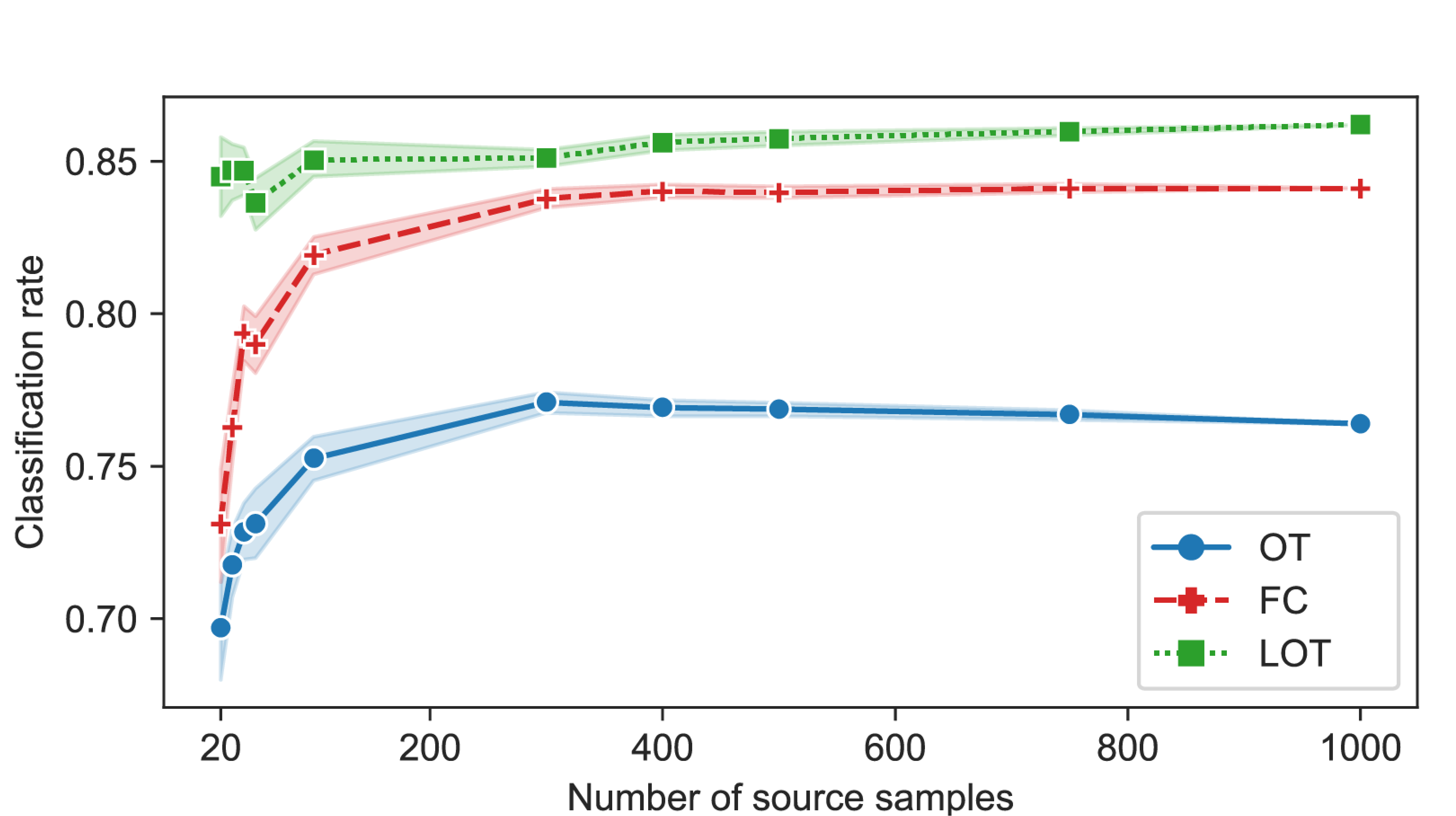}}
    \caption{\footnotesize {\bf{\em \alg provides robust alignment, even when given very few samples.}}  We compare our method with OT and \fc on the MNIST-USPS domain adaptation task when different numbers of USPS samples are available. Reported classification rates are averaged over 50 random sets.
    \label{fig:robustness}\vspace{-5mm}}
\end{figure}

The decomposition in both \alg and \fc allows us to visualize transport between the source, anchors, and the target (Figure~\ref{fig:mnist}c-d, 
\ref{fig:allsankey}). 
This visualization highlights the interpretability of the transport plans learned via our approach, with the middle transport plan ${\bf P}_z$ providing a concise map of interactions between class manifolds in the unbalanced setting. With \alg (Figure~\ref{fig:mnist}c), we find that each source anchor is mapped to the correct target anchor, with some minor interactions with the target anchors corresponding to the removed digits. In comparison, \fc (Figure~\ref{fig:mnist}d, \ref{fig:allsankey}) has more spurious interactions between source, anchors, and target.

{\bf E3) Robustness to sampling:}~We examined the robustness to sampling for the MNIST to USPS example (Figure~\ref{fig:robustness}). In this case, we find that \alg has a stable alignment as we subsample the source dataset, with very little degradation in classification accuracy, even as we reduce the source to only 20 samples. We also observe a significant gap between \alg and other approaches in this experiment, with more than a 10\% gap between \fc and \alg when very few source samples are provided. Our results demonstrate that \alg is robust to subsampling, providing empirical evidence for Proposition~\ref{prop:sampling}.

\section{Discussion}\label{sec:discussion}
In this paper, we introduced \alg\hspace{-1mm}, a new form of structured transport leading to an approach for jointly clustering and aligning data. We provided an efficient optimization method to solve for the transport, and studied its statistical rate via theoretical analysis and robustness to data perturbations with empirical experiments. In the future, we would like to explore the application of \alg to non-Euclidean spaces, and incorporate metric learning into our framework.

\section*{Acknowledgements}

This project is supported by the NIH award 1R01EB029852-01, NSF awards IIS-1755871 and CCF-1740776, as well as generous gifts from the Alfred Sloan Foundation and the McKnight Foundation.

\bibliographystyle{icml2021}
\bibliography{cbo.bib}

\onecolumn
\newpage
\section*{\Large Appendix}

\renewcommand{\thefigure}{S\arabic{figure}}
\setcounter{figure}{0}
\setcounter{table}{0}
\renewcommand{\thetable}{S\arabic{table}}%

\appendix

\section{Detailed proofs}\label{sec:suppaproof}
\subsection{Proof of Proposition \ref{prop1}}
\begin{repproposition}{prop1}
Suppose the latent spaces ${\cal{Z}}_x={\cal{Z}}_y$ are the same as the original data spaces ${\cal{X}}={\cal{Y}}$, and the cost matrices are defined by $\mathbf{C}_x[a,b] = \mathbf{C}_z[a,b]= \mathbf{C}_y[a,b]= d(a,b)^p$, where $p\geq1$ and $d$ is some distance function.
If we define the latent Wasserstein discrepancy as ${\cal{W}}^L_p:={(\rm{OT}^L)}^{1/p}$, then there exist $\kappa>0$ such that, for any $\mu$, $\nu $ and $\zeta$ having latent distributions of support sizes up to $k$, the discrepancy satisfies,
\begin{itemize}
    \item ${\cal{W}}_p^L(\mu,\nu)\geq 0,$ (Non-negativity)
    \item ${\cal{W}}_p^L(\mu,\nu) = {\cal{W}}_p^L(\nu,\mu),$ (Symmetry)
    \item $\exists \kappa >0$ such that ${\cal{W}}_p^L(\mu,\nu) \leq \kappa \left({\cal{W}}_p^L(\mu,\zeta)+{\cal{W}}_p^L(\zeta,\nu)\right),$ (Quasi-Triangle inequality)
\end{itemize}
\end{repproposition}
\begin{proof}
The first two properties are easy to verify by non-negativity and symmetry of the Wasserstein distance. Hence we will prove the last property.
For brevity, we denote ${\cal{Z}}={\cal{Z}}_x={\cal{Z}}_y$.
Under the assumptions on the cost matrices, we have ${\rm OT}_{\mathbf{C}}={\cal{W}}^p_p$. Hence there exist latent distributions $\mu_z^*,\zeta_z^*,\zeta_z'^*,\nu_z^*\in\Delta^{k}_{{\cal{Z}}}$ satisfy,
\begin{align}
            {(\rm{OT}^L)}(\mu,\zeta) &= {\cal{W}}^p_p(\mu,\mu_z^*)+{\cal{W}}^p_p(\mu_z^*,\zeta_z^*) + {\cal{W}}^p_p(\zeta_z^*,\zeta),\nonumber\\
    {(\rm{OT}^L)}(\zeta,\nu) &= {\cal{W}}^p_p(\zeta,{\zeta'}^*_z)+{\cal{W}}^p_p({\zeta'}^*_z,\nu_z^*) + {\cal{W}}^p_p(\nu_z^*,\nu).\nonumber
\end{align}
As $d$ is a distance function, we know that ${\cal{W}}_p$ is a metric \cite{peyre2019computational} that satisfies the triangle inequality:
\begin{align}\label{tri}
    {\cal{W}}_p(\mu,\zeta_z'^*)\leq {\cal{W}}_p(\mu,\mu_z^*) + {\cal{W}}_p(\mu_z^*,\zeta_z^*)+{\cal{W}}_p(\zeta^*_z,\zeta) + {\cal{W}}_p(\zeta,{\zeta'}^*_z).
\end{align}
On the other hand, Jensen's inequality tells us that $(a+b+c+d)^p\leq 4^{p-1}(a^p+b^p+c^p+d^p)$ for any non-negative $a,b,c,d$. We apply this inequality to (\ref{tri}) and get,
\begin{align}\label{jentri}
     {\cal{W}}^p_p(\mu,{\zeta'}^*_z)\leq 4^{p-1}({\cal{W}}^p_p(\mu,\mu_z^*) + {\cal{W}}^p_p(\mu_z^*,\zeta_z^*)+{\cal{W}}^p_p(\zeta^*_z,\zeta) + {\cal{W}}^p_p(\zeta,{\zeta'}^*_z)).
\end{align}
Thus, 
\begin{align}
    &{\cal{W}}_p^L(\mu,\nu)=\left({(\rm{OT}^L)}(\mu,\nu)\right)^{\frac{1}{p}}\leq ({\cal{W}}^p_p(\mu,{\zeta'}^*_z) + {\cal{W}}^p_p({\zeta'}^*_z,\nu_z^*) +{\cal{W}}^p_p(\nu_z^*,\nu))^{\frac{1}{p}}\nonumber  \\
    \overset{(\ref{jentri})}{\leq} & 4^{1-\frac{1}{p}}({\cal{W}}^p_p(\mu,\mu_z^*) + {\cal{W}}^p_p(\mu_z^*,\zeta_z^*)+{\cal{W}}^p_p(\zeta^*_z,\zeta) + {\cal{W}}^p_p(\zeta,{\zeta'}^*_z)+{\cal{W}}^p_p({\zeta'}^*_z,\nu_z^*) + {\cal{W}}^p_p(\nu_z^*,\nu))^{\frac{1}{p}}\nonumber\\
    \leq & 4^{1-\frac{1}{p}}({\cal{W}}^p_p(\mu,\mu_z^*) + {\cal{W}}^p_p(\mu_z^*,\zeta_z^*)+{\cal{W}}^p_p(\zeta^*_z,\zeta))^{\frac{1}{p}} +( {\cal{W}}^p_p(\zeta,{\zeta'}^*_z)+{\cal{W}}^p_p({\zeta'}^*_z,\nu_z^*) + {\cal{W}}^p_p(\nu_z^*,\nu))^{\frac{1}{p}},\nonumber
\end{align}
where the last inequality follows from $(a+b)^{\frac{1}{p}}\leq a^{\frac{1}{p}}+b^{\frac{1}{p}}$ for any nonnegative $a,b$.
Choosing $\kappa =  4^{1-\frac{1}{p}}$ completes the proof.
\end{proof}

\subsection{Proof of Corollary \ref{cor:kmean}}
\begin{proof}
Without ambiguity the optimizations in the followings are over $\Delta_{{\cal{Z}}}^k$.
Recall the definition, \begin{align}
&\tilde{{\cal{W}}}_p^L(\mu,\nu):=\left(\left({\cal{W}}_p^L(\mu,\nu)\right)^p-\min\limits_{z_k}{\cal{W}}_p^p(\mu,z_k)-\min\limits_{z'_k}{\cal{W}}_p^p(\nu,z'_k)\right)^{1/p}\\
=&\left(\min\limits_{\mu_z,\nu_z}
\left({\cal{W}}_p^p(\mu,\mu_z)+{\cal{W}}_p^p(\mu_z,\nu_z)+{\cal{W}}_p^p(\nu_z,\nu)\right)
-\min\limits_{z_k}{\cal{W}}_p^p(\mu,z_k)-\min\limits_{z'_k}{\cal{W}}_p^p(\nu,z'_k)
\right)^{1/p}\\
= &\left(\left({\cal{W}}_p^p(\mu,\mu^*_z)-\min\limits_{z_k}{\cal{W}}_p^p(\mu,z_k)\right)+{\cal{W}}_p^p(\mu^*_z,\nu^*_z) + \left({\cal{W}}_p^p(\nu,\nu^*_z)-\min\limits_{z'_k}{\cal{W}}_p^p(\nu,z'_k)\right)\right)^{1/p}
\geq {\cal{W}}_p(\mu^*_z,\nu^*_z)\geq0\label{cor:ineq},
\end{align}
where $(\mu^*_z,\nu_z^*)\in \arg\min\limits_{\mu_z,\nu_z}
\left({\cal{W}}_p^p(\mu,\mu_z)+{\cal{W}}_p^p(\mu_z,\nu_z)+{\cal{W}}_p^p(\nu_z,\nu)\right)$. Observe that the three terms in (\ref{cor:ineq}) are non-negative, thus $\tilde{{\cal{W}}}_p^L(\mu,\nu)=0$ only if the followings are simultaneously satisfied,
\begin{align}
   &{\cal{W}}_p^p(\mu,\mu^*_z)=\min\limits_{z_k}{\cal{W}}_p^p(\mu,z_k),\\&{\cal{W}}_p^p(\nu,\nu^*_z)=\min\limits_{z'_k}{\cal{W}}_p^p(\nu,z'_k),\\
   &{\cal{W}}_p^p(\mu^*_z,\nu^*_z)=0.\label{ceneq}
\end{align}
The last condition implies $\mu^*_z=\nu_z^*=:\xi_z$, so the other conditions become ${\cal{W}}_p^p(\mu,\xi_z)=\min\limits_{z_k}{\cal{W}}_p^p(\mu,z_k)$ and ${\cal{W}}_p^p(\mu,\xi_z)=\min\limits_{z'_k}{\cal{W}}_p^p(\nu,z'_k)$. 
We recognize that $\min\limits_{z_k}{\cal{W}}_p^p(\mu,z_k)$ is a Wasserstein barycenter problem such that the barycenter supports on up to $k$ locations. So the conditions imply that the barycenters of $\mu$ and $\nu$ must coincide.
On the other hand, when $p=2$, \cite{peyre2019computational,ho2017multilevel} show that the solution $z_k$ to the barycenter problem is the distribution of k-means centroids of $\mu$ with a distribution proportional to the mass of clusters. The same conclusion also applies to $\nu$. Hence, the condition (\ref{ceneq})
shows that not only the centroids of k-means to $\mu$ and $\nu$ must be the same, but also the proportions of their corresponding clusters must be equal, which completes the proof of the corollary.
\end{proof}
\subsection{Proof of Proposition \ref{prop:relax}}
The proof of Proposition \ref{prop:relax} will rely on the following lemma.
\begin{lemma}\label{logsum}
(Log-sum inequality) : Let $x_m$ and $y_m$, $m=1,\dots, n$, be nonnegative sequences, then
\begin{align}
    \left(\sum_{m=1}^n x_m\right)\log\left(\frac{\sum_{m=1}^n x_m}{\sum_{m=1}^n y_m}\right)\leq \sum_{m=1}^n x_m\log\left(\frac{ x_m}{ y_m}\right).
\end{align}
\end{lemma}
\begin{proof}
Denote $x=\sum_m x_m$ and $y = \sum_my_m$. By concavity of the $\log $ function and Jensen's inequality, we have,
\begin{align}
    \sum_m{\frac{x_m}{x}}\log\left(\frac{y_m}{x_m}\right)\leq \log\left(\sum_m \frac{x_m}{x}\frac{y_m}{x_m}\right)=\log\left(\frac{y}{x}\right).
\end{align}
Multiplying both sides of the above inequality with $-x$ yields the lemma.
\end{proof}
\begin{repproposition}{prop:relax}
 Let $\bf P$ be transport plan of the form in (\ref{form}). Assume $\mathbf{K}$ is some Gibbs kernel satisfying $\mathbf{K}_x\mathbf{K}_z\mathbf{K}_y\leq \mathbf{K}$, where the inequality is over each entry. The following inequality holds, 
\begin{align}\label{upper}
    \varepsilon \text{KL}(\mathbf{P}\|\mathbf{K}) \leq \varepsilon (\text{KL}(\mathbf{P}_x\|\mathbf{K}_x) +\text{KL}(\mathbf{P}_z\|\mathbf{K}_z)
    + \text{KL}(\mathbf{P}_y\|\mathbf{K}_y)+\mathbf{H}(\mathbf{u}_z)+ \mathbf{H}(\mathbf{v}_z)),
\end{align}
where $\mathbf{H}(\mathbf{a}):=-\sum_i \mathbf{a}_i\log \mathbf{a}_i$ denotes the Shannon entropy.
\end{repproposition}
\begin{proof}
As the transport cost is monotonically decreasing in the entries of $\mathbf{K}$, we only need to prove the case when $\mathbf{K}_x\mathbf{K}_z\mathbf{K}_y= \mathbf{K}$.
To simplify notations, we define $\tilde{\mathbf{P}}_z:=\text{diag}(\mathbf{u}_z^{-1})\mathbf{P}_z\text{diag}(\mathbf{v}_z^{-1})$, $\tilde{\mathbf{P}}_y:=\text{diag}(\mathbf{v}_z^{-1})\mathbf{P}_y$. Now by definition of $\mathbf{P}$ with the form (\ref{form}) we have,
\begin{align}\label{pa}
    \text{KL}(\mathbf{P}\|\mathbf{K}) = \sum_{i,j}\left(\sum_{m}\left((\mathbf{P}_x)_{i,m}(\tilde{\mathbf{P}}_z{\mathbf{P}}_y)_{m,j}\right)\log{\frac{\sum_m(\mathbf{P}_x)_{i,m}(\tilde{\mathbf{P}}_z{\mathbf{P}}_y)_{m,j}}{\sum_m(\mathbf{K}_x)_{i,m}(\mathbf{K}_z\mathbf{K}_y)_{m,j}}}\right)=:\sum_{i,j}a_{i,j}.
\end{align}
Apply Lemma \ref{logsum} to $a_{i,j}$ with $x_m=(\mathbf{P}_x)_{i,m}(\tilde{\mathbf{P}}_z\mathbf{P}_y)_{m,j}$ and $y_m=\sum_m(\mathbf{K}_x)_{i,m}(\mathbf{K}_z\mathbf{K}_y)_{m,j}$ we have,
\begin{align}\label{pab}
 a_{i,j} \leq \sum_m (\mathbf{P}_x)_{i,m}(\tilde{\mathbf{P}}_z\mathbf{P}_y)_{m,j}\log{\frac{(\mathbf{P}_x)_{i,m}}{(\mathbf{K}_x)_{i,m}}} + \sum_m (\mathbf{P}_x)_{i,m}(\tilde{\mathbf{P}}_z\mathbf{P}_y)_{m,j}\log{\frac{(\tilde{\mathbf{P}}_z\mathbf{P}_y)_{m,j}}{(\mathbf{K}_z\mathbf{K}_y)_{m,j}}}=:b_{i,j} + c_{i,j}. 
\end{align}
Since $\tilde{\mathbf{P}}_z\mathbf{P}_y\mathbf{1} = \text{diag}(\mathbf{u}_z^{-1})\mathbf{P}_z\text{diag}(\mathbf{v}_z^{-1})\mathbf{v}_z=\text{diag}(\mathbf{u}_z^{-1})\mathbf{u}_z=\mathbf{1}$, we have $\sum_{j}(\tilde{\mathbf{P}}_z\mathbf{P}_y)_{m,j}=1$, $\forall m$, and 
\begin{align}\label{pb}\sum_{i,j}b_{i,j}=\sum_i\sum_m (\mathbf{P}_x)_{i,m}\left(\sum_j(\tilde{\mathbf{P}}_z\mathbf{P}_y)_{m,j}\right)\log{\frac{(\mathbf{P}_x)_{i,m}}{(\mathbf{K}_x)_{i,m}}}=\sum_{i,m} (\mathbf{P}_x)_{i,m}\log{\frac{(\mathbf{P}_x)_{i,m}}{(\mathbf{K}_x)_{i,m}}}=\text{KL}(\mathbf{P}_x\|\mathbf{K}_x).
\end{align}
On the other hand, $\mathbf{P}_x^T\mathbf{1}=\mathbf{u}_z$ implies $\sum_{i}(\mathbf{P}_x)_{i,m}=(\mathbf{u}_z)_m$, $\forall m$, and 
\begin{align}
    \sum_{i,j}c_{i,j} =& \sum_{m,j}(\mathbf{u}_z)_m(\tilde{\mathbf{P}}_z\mathbf{P}_y)_{m,j}\log{\frac{(\tilde{\mathbf{P}}_z\mathbf{P}_y)_{m,j}}{(\mathbf{K}_z\mathbf{K}_y)_{m,j}}}=\sum_{m,j}({\mathbf{P}}_z\text{diag}(\mathbf{v}_z^{-1})\mathbf{P}_y)_{m,j}\log{\frac{({\mathbf{P}}_z\text{diag}(\mathbf{v}_z^{-1})\mathbf{P}_y)_{m,j}}{(\mathbf{u}_z)_m(\mathbf{K}_z\mathbf{K}_y)_{m,j}}}\nonumber\\
    \overset{(i)}{=}&\mathbf{H}(\mathbf{u}_z) + \sum_{m,j}({\mathbf{P}}_z\text{diag}(\mathbf{v}_z^{-1})\mathbf{P}_y)_{m,j}\log{\frac{({\mathbf{P}}_z\text{diag}(\mathbf{v}_z^{-1})\mathbf{P}_y)_{m,j}}{(\mathbf{K}_z\mathbf{K}_y)_{m,j}}}\nonumber\\
    \overset{(ii)}{=}&\mathbf{H}(\mathbf{u}_z) + \sum_{m,j}({\mathbf{P}}_z\tilde{\mathbf{P}}_y)_{m,j}\log{\frac{({\mathbf{P}}_z\tilde{\mathbf{P}}_y)_{m,j}}{(\mathbf{K}_z\mathbf{K}_y)_{m,j}}}=:\mathbf{H}(\mathbf{u}_z) + \sum_{m,j}d_{m,j}\label{pc},
\end{align}
where (i) follows from $\mathbf{P}_z\text{diag}(\mathbf{v}_z^{-1})\mathbf{P}_y\mathbf{1}=\mathbf{u}_z$ and (ii) from the definition of $\tilde{\mathbf{P}}_y$. Now applying Lemma~\ref{logsum} again to $d_{m,j}$ with $x_l=(\mathbf{P}_z)_{m,l}(\tilde{\mathbf{P}}_y)_{l,j},~y_l=(\mathbf{K}_z)_{m,l}(\mathbf{K}_y)_{l,j}$ leads to
\begin{align}
    d_{m,j}\leq& \sum_l \left((\mathbf{P}_z)_{m,l}(\tilde{\mathbf{P}}_y)_{l,j}\right)\log\frac{(\mathbf{P}_z)_{m,l}(\tilde{\mathbf{P}}_y)_{l,j}}{(\mathbf{K}_z)_{m,l}(\mathbf{K}_y)_{l,j}}\nonumber\\
    =&\sum_l\left((\mathbf{P}_z)_{m,l}(\tilde{\mathbf{P}}_y)_{l,j}\right)\log\frac{(\mathbf{P}_z)_{m,l}}{(\mathbf{K}_z)_{m,l}} + \sum_l\left((\mathbf{P}_z)_{m,l}(\tilde{\mathbf{P}}_y)_{l,j}\right)\log\frac{(\tilde{\mathbf{P}}_y)_{l,j}}{(\mathbf{K}_y)_{l,j}}=:e_{m,j} + f_{m,j}\label{pd}.
\end{align}
From $\tilde{\mathbf{P}}_y\mathbf{1} =\text{diag}(\mathbf{v}_z^{-1})\mathbf{P}_y\mathbf{1}=(\mathbf{v}_z^{-1})\mathbf{v}_z=\mathbf{1}$, we get $\sum_{j}({\tilde{\mathbf{P}}}_y)_{l,j}=1$ and
\begin{align}\label{pe}
    \sum_{m,j}e_{m,j} =\sum_m(\mathbf{P}_z)_{m,l}\sum_{j}\left((\tilde{\mathbf{P}}_y)_{l,j}\right)\log\frac{(\mathbf{P}_z)_{m,l}}{(\mathbf{K}_z)_{m,l}}= \text{KL}(\mathbf{P}_z\|\mathbf{K}_z).
\end{align}
Similarly, using $\mathbf{P}_z^T\mathbf{1} = \mathbf{v_z}$, we also get $\sum_m(\mathbf{P}_z)_{m,l}=(\mathbf{v}_z)_l$, $\forall l$, and
\begin{align}
    \sum_{m,j}f_{mj} =& -\sum_{m,j,l}(\mathbf{P}_z)_{m,l}(\tilde{\mathbf{P}}_y)_{l,j}\log(\mathbf{v}_z)_l +\sum_{m,j,l}(\mathbf{P}_z)_{m,l}(\tilde{\mathbf{P}}_y)_{l,j}\log\frac{(\mathbf{P}_y)_{l,j}}{(\mathbf{K}_y)_{l,j}}\nonumber\\
    =&~\mathbf{H}(\mathbf{v}_z) + \text{KL}(\mathbf{P}_y\|\mathbf{K}_y)\label{pf},
\end{align}
where the last equality follows from $\sum_{j}({\tilde{\mathbf{P}}}_y)_{l,j}=1$ and $\sum_m(\mathbf{P}_z)_{m,l}(\tilde{\mathbf{P}}_y)_{l,j}=(\mathbf{v}_z)_l(\mathbf{v}_z)^{-1}_l({\mathbf{P}}_y)_{l,j}$.
Finally, combining equations from (\ref{pa}-\ref{pf}) completes the proof.
\end{proof}

\subsection{Proof of Corollary \ref{cor:relate}}
\begin{proof}
By the assumptions on the cost matrices, the condition $\mathbf{K}_x\mathbf{K}_z\mathbf{K}_y\leq \mathbf{K}$ is equivalent to
\begin{align}\label{concor}
    \|\x_i-\y_j\|_p^p \leq -\varepsilon\log\left(\sum_{l,m}\exp\left(\frac{-1}{\varepsilon}\left(3^{p-1}\|\x_i-\z^x_l\|_p^p+3^{p-1}\|\z_l^x-\z^y_m\|_p^p+3^{p-1}\|y_j-\z^y_m\|_p^p\right)\right)\right).
\end{align}
Observe on the right-hand side is nothing more than the soft-min of the set \begin{align}
    {\cal{A}}:=\{3^{p-1}\|\x_i-\z^x_l\|_p^p+3^{p-1}\|\z_l^x-\z^y_m\|_p^p+3^{p-1}\|y_j-\z^y_m\|_p^p:\forall l,m\},
\end{align} where soft-min$_{\varepsilon}({\cal{A}}):=-\varepsilon\log(\sum_{\alpha\in {\cal{A}}}\exp\left(-\alpha/\varepsilon)\right)$, which approaches the minimum of the set $\min {\cal{A}}$ as $\varepsilon\to 0$. On the other hand, any element in the set is at least $\|\x_i-\y_j\|_p^p$ because
\begin{align}
   \|\x_i-\y_j\|_p^p\overset{(i)}{\leq} &~ (\|\x_i-\z_l^x\|_p + \|\z_l^x-\z_m^y\|_p + \|\y_j-\z_m^y\|_p)^p\nonumber\\
   \overset{(ii)}{\leq} &~ 3^{p-1}(\|\x_i-\z_l^x\|_p + \|\z_l^x-\z_m^y\|_p + \|\y_j-\z_m^y\|_p),
\end{align}
where (i) follows from Minkowski inequality and (ii) from applying Jensen's inequality to convex function $x^p$.
Hence, (\ref{concor}) holds for sufficient small $\varepsilon$, and thus \alg with these cost matrices gives a transport plan according to the upper bound in proposition \ref{prop:relax}.
\end{proof}

\subsection{Proof of Proposition \ref{prop:sampling}}
\begin{repproposition}{prop:sampling}
Suppose $X$ and $Y$ have distributions $\mu$ and $\nu$ supported on a compact region $\Omega$ in $\mathbb{R}^d$, the cost functions $c_x(\cdot,\cdot)$ and $c_y(\cdot,\cdot)$ are defined as the squared Euclidean distance, and $\hat{\mu}$, $\hat{\nu}$ are empirical distributions of $n$ and $m$ i.i.d. samples from $\mu$ and $\nu$, respectively. If the spaces for latent distributions are equal to ${\cal{Z}}_x={\cal{Z}}_y=\mathbb{R}^d$, and there are $k_x$ and $k_y$ anchors in the source and target, respectively, then with probability at least $1-\delta$,   
\begin{align}\label{prop:sample}
    {\rm Err}  ~ \leq~ &C\sqrt{\frac{k_{max}^3d \log k_{max}+ \log (2/\delta)}{N}},
\end{align}
where ${\rm Err} = |{\rm OT}^L(\mu,\nu) - {\rm OT}^L(\hat{\mu},\hat{\nu})|$, $k_{max}=\max\{k_x,k_y\}$ and $N= \min\{n,m\}$ and $C\geq 0$ is some constant not depending on $N$.
\end{repproposition}
The proof is an application of the following Lemma.
\begin{lemma}(Theorem 4 in~ \cite{forrow2019statistical}).
Suppose $\mu$ is a distribution supported in a compact region $K\subseteq \mathbb{R}^d$. Let $\hat{\mu}$ be an empirical distribution of $n$ i.i.d. samples from $\mu$, then there exists a constant $C>0$ such that for any distribution $\mu_z$ supported on up to $k$ points, with probability $1-\delta$, we have,
\begin{align}
    |{\cal{W}}_2^2(\mu,\mu_z) -{\cal{W}}_2^2(\hat{\mu},\mu_z)| \leq C\sqrt{\frac{k^3d\log k + \log(2/\delta)}{n}}.
\end{align}
\end{lemma}
\begin{proof}(of Proposition \ref{prop:sampling}).
Following the assumptions of the proposition, $OT^L(\mu,\nu)$, $OT^L(\hat{\mu},\hat{\nu})$ can be rewritten into optimizations involving the 2-Wasserstein distance,
\begin{align}\label{opt1}
    {\rm OT}^L(\mu,\nu) = \min_{\mu_z\in \Delta_{\mathbb{R}^d}^{k_x},\nu_z\in \Delta_{\mathbb{R}^d}^{k_y}} {{\cal{W}}}_2^2(\mu,\mu_z) + {{\cal{W}}}_2^2(\mu_z,\nu_z) + {{\cal{W}}}_2^2(\nu_z,\nu), \\
    {\rm OT}^L(\hat{\mu},\hat{\nu}) = \min_{\mu_z\in \Delta_{\mathbb{R}^d}^{k_x},\nu_z\in \Delta_{\mathbb{R}^d}^{k_y}} {{\cal{W}}}_2^2(\hat{\mu},\mu_z) + {{\cal{W}}}_2^2(\mu_z,\nu_z) + {{\cal{W}}}_2^2(\nu_z,\hat{\nu}).
\end{align}
Without lose of generality, we assume that ${\rm OT}^L(\mu,\nu)\geq {\rm OT}^L(\hat{\mu},\hat{\nu})$. Let $(\mu_z^*,\nu_z^*)$ and  $(\hat{\mu}_z^*,\hat{\nu}_z^*)$ denote the optimal solutions to the two optimization problems. Then,
\begin{align*}
    &~{\rm OT}^L(\mu,\nu)-{\rm OT}^L(\hat{\mu},\hat{\nu})\nonumber\\
    =&~ {{\cal{W}}}_2^2(\mu,\mu_z^*) + {{\cal{W}}}_2^2(\mu_z^*,\nu_z^*) + {{\cal{W}}}_2^2(\nu_z^*,\nu) - \left({{\cal{W}}}_2^2(\hat{\mu},\hat{\mu}_z^*) + {{\cal{W}}}_2^2(\hat{\mu}_z^*,\hat{\nu}_z^*) + {{\cal{W}}}_2^2(\hat{\nu}_z^*,\hat{\nu})\right)\\
    \overset{(a)}{\leq}~& {{\cal{W}}}_2^2(\mu,\hat{\mu}_z^*) + {{\cal{W}}}_2^2(\hat{\mu}_z^*,\hat{\nu}_z^*) + {{\cal{W}}}_2^2(\hat{\nu}_z^*,\nu) - \left({{\cal{W}}}_2^2(\hat{\mu},\hat{\mu}_z^*) + {{\cal{W}}}_2^2(\hat{\mu}_z^*,\hat{\nu}_z^*) + {{\cal{W}}}_2^2(\hat{\nu}_z^*,\hat{\nu})\right)\\
    \leq~& |{{\cal{W}}}_2^2(\mu,\hat{\mu}_z^*)-{{\cal{W}}}_2^2(\hat{\mu},\hat{\mu}_z^*)| +  |{{\cal{W}}}_2^2(\nu,\hat{\nu}_z^*)-{{\cal{W}}}_2^2(\hat{\nu},\hat{\nu}_z^*)|\\
    \overset{(b)}{\leq} ~& C_1\sqrt{\frac{k_x^3 d \log k_x +\log(2/\delta)}{n}} + C_2\sqrt{\frac{k_y^3 d \log k_y +\log(2/\delta)}{m}}\\
    \leq ~& (C_1+C_2)\sqrt{\frac{k_{\rm max}^3 d \log k_{\rm max} +\log(2/\delta)}{N}},
\end{align*}
where $(a)$ follows from the optimality of $(\mu_z^*,\nu_z^*)$ to the objective (\ref{opt1}), (b) by applying Lemma 1 twice. By symmetry, we get a similar result for the case where ${\rm OT}^L(\mu,\nu)\leq {\rm OT}^L(\hat{\mu},\hat{\nu})$ which completes the proof.
\end{proof}


\begin{algorithm}[t]
\begin{algorithmic}[1]
\caption{Iterative Bregman Projection}\label{alg:ibp}
\INPUT{$\mathbf{K}$} \OUTPUT{$\mathbf{P}$} 
\ENSURE{ $\mathbf{P}\leftarrow \mathbf{K} $ }
\WHILE{not converging}
\FOR{$i=1,\dots,k$}
\STATE{$\mathbf{P}\leftarrow \Pi_{{\cal{C}}_i}^{\text{KL}}(\mathbf{P})$}
\ENDFOR
\ENDWHILE
\end{algorithmic}
\end{algorithm}
\section{Additional algorithms and derivations}
\subsection{Derivation of Algorithm \ref{alg:1}}\label{der1}
In this section, we provide the detailed derivation of the algorithm \ref{alg:1}.
It is based on the Dykstra's algorithm \cite{dykstra1983algorithm} which considers problems of the form:
\begin{align}\label{bregform}
    \min_{\mathbf{P}\in \mathbf{R}^{n\times m}} \text{KL}(\mathbf{P}||\mathbf{K})\nonumber\\
    \text{subject to }\mathbf{P}\in {\bigcap_{i=1}^k{\cal{C}}_i}, 
\end{align}
where $\mathbf{K}\in \mathbb{R}_+^{n\times m}$ is some non-negative fixed matrix and ${\cal{C}}_i$, $\forall i$ are closed convex sets. When in addition ${\cal{C}}_i$ are affine, an iterative Bregman projection \cite{benamou2015iterative} solves the problem and has the form of algorithm \ref{alg:ibp}. Here $\Pi_{{\cal{C}}}^{\text{KL}}(\mathbf{P})$ denotes the Bregman projection of $\mathbf{P}$ onto $\cal{C}$ defined as $\Pi_{{\cal{C}}}^{\text{KL}}(\mathbf{P}):=\argmin_{\mathbf{\gamma}\in {\cal{C}}} \text{KL}(\gamma||\mathbf{P})$.

Recall that the objective of algorithm \ref{alg:1} is 
\begin{align}\label{cost1for2}
    \min_{\mathbf{P}_x,\mathbf{P}_z,\mathbf{P}_y} &\varepsilon(\text{KL}(\mathbf{P}_x\|\mathbf{K}_x)+\text{KL}(\mathbf{P}_z\|\mathbf{K}_z)\nonumber
    +\text{KL}(\mathbf{P}_y\|\mathbf{K}_y)),
    \nonumber\\ \text{subject to:~ }&\exists \mathbf{u}_z,\mathbf{v}_z : \mathbf{P}_x\mathbf{1} =\mu,\mathbf{P}^T_x\mathbf{1}=\mathbf{u}_z,\mathbf{P}_z\mathbf{1} = \mathbf{u}_z,
    \mathbf{P}_z^T\mathbf{1} =\mathbf{v}_z,\mathbf{P}_y\mathbf{1}=\mathbf{v}_z,\mathbf{P}^T_y\mathbf{1}=\nu. 
\end{align}
Hence, we can write (\ref{cost1for2}) into the form of (\ref{bregform}) by defining 
\begin{align}\label{affineset}
&{\cal{C}}_1=\{\mathbf{P}_x:\mathbf{P}_x\mathbf{1} =\mu\}, {\cal{C}}_2=\{\mathbf{P}_y:\mathbf{P}_y^T\mathbf{1}=\nu\},\nonumber\\
&{\cal{C}}_3=\{(\mathbf{P}_x,\mathbf{P}_z):\exists \mathbf{v}_z,\mathbf{P}^T_x\mathbf{1} =\mathbf{P}_z\mathbf{1} = \mathbf{u}_z\},{\cal{C}}_4=\{(\mathbf{P}_z,\mathbf{P}_y):\exists \mathbf{u}_z,\mathbf{P}^T_z\mathbf{1} =\mathbf{P}_y\mathbf{1} = \mathbf{v}_z\}.
\end{align}
For further simplifications, observe that each transport matrix is a Bregman projection per iteration, so it can be easily verified that the transport plan must be of the form $\mathbf{P}_i =\text{diag}(\alpha_i)\mathbf{K}_i\text{diag}(\beta_i)$, for some $\alpha_i\in\mathbb{R}_+^n,\beta_i\in\mathbb{R}_+^m $, $i\in\{x,y,z\}$. Now, using a standard technique of Lagrange multipliers, the projections to each of these sets (\ref{affineset}) can be written as
\begin{align*}
   &\Pi_{{\cal{C}}_1}^{\text{KL}}(\mathbf{P}_x) =\text{diag}(\mu\oslash \mathbf{K}_x\beta)\mathbf{K}_x\text{diag}(\beta_x),\\
   &\Pi_{{\cal{C}}_2}^{\text{KL}}(\mathbf{P}_y) =\text{diag}(\alpha_y)\mathbf{K}_y\text{diag}(\nu\oslash\mathbf{K}_y^T\alpha_y),\\
 &\Pi_{{\cal{C}}_3}^{\text{KL}}(\mathbf{P}_x) =\text{diag}(\alpha_x)\mathbf{K}_x\text{diag}(\mathbf{u}_z\oslash \mathbf{K}_x^T\alpha_x),~
 \Pi_{{\cal{C}}_3}^{\text{KL}}(\mathbf{P}_z) =\text{diag}(\mathbf{u}_z\oslash \mathbf{K}_z\beta_z)\mathbf{K}_x\text{diag}(\beta_z),\\
 &\text{where }\mathbf{u}_z = ((\alpha_z\odot \mathbf{K}_z\beta_z)\odot(\beta_x\odot \mathbf{K}_x^T\alpha_x))^{\frac{1}{2}}\\
&\Pi_{{\cal{C}}_4}^{\text{KL}}(\mathbf{P}_z) =\text{diag}(\alpha_z)\mathbf{K}_z\text{diag}(\mathbf{v}_z\oslash \mathbf{K}_z^T\alpha_z),~
 \Pi_{{\cal{C}}_4}^{\text{KL}}(\mathbf{P}_y) =\text{diag}(\mathbf{v}_z\oslash \mathbf{K}_y\beta_y)\mathbf{K}_y\text{diag}(\beta_y),\\
 &\text{where }\mathbf{v}_z = ((\alpha_y\odot \mathbf{K}_y\beta_y)\odot(\beta_z\odot \mathbf{K}_z^T\alpha_z))^{\frac{1}{2}}.
\end{align*}
Finally, by keeping track with $\alpha_i,\beta_i$, we arrive at the algorithm \ref{alg:1}.

\subsection{Rule of anchor update}\label{b:anchoropt}
When $\mathbf{C}_{x}= [d_{\mathbf{M}_x}(\x_i,\z^x_m)]_{i,m},  \mathbf{C}_{z}:=[d_{\mathbf{M}_z}(\z^x_m,\z^y_n)]_{m,n}, \mathbf{C}_{y} = [d_{\mathbf{M}_y}(\x_j,\z^y_n)]_{n,j}$, 
we can rewrite the objective of \alg explicitly as a function of $\mathbf{Z}_x$ and $\mathbf{Z}_y$:
\begin{align}\label{xform}
  \text{OT}^{\text{L}} =   &\text{tr}\left(\mathbf{Z}_x^T(\mathbf{M}_x+\mathbf{M}_y)\mathbf{Z}_x\text{diag}(\mathbf{u}_z)\right)
    +\text{tr}\left(\mathbf{Z}_y^T(\mathbf{M}_y+\mathbf{M}_z)\mathbf{Z}_y\text{diag}(\mathbf{v}_z)\right)\nonumber\\
    -&2\text{tr}\left(\mathbf{P}_x^T\mathbf{X}^T\mathbf{M}_x\mathbf{Z}_x\right)-2\text{tr}\left(\mathbf{P}_z^T\mathbf{Z}_x^T\mathbf{M}_z\mathbf{Z}_y)
    -2\text{tr}(\mathbf{P}_y^T\mathbf{Z}_y^T\mathbf{M}_y\mathbf{Y}\right).
\end{align}
We get the first-order stationary point by setting the gradient of ${\rm OT^\text{L}}$ with respect to $\mathbf{Z}_x$ and $\mathbf{Z}_y$ to zero, which results in the closed-form formula,

\begin{align}\label{cenfor}
\begin{bmatrix}\text{vec}(\mathbf{Z}_x^*)\\\text{vec}(\textbf{Z}_y^*)
    \end{bmatrix}
    =&\begin{bmatrix}
\text{diag}(\mathbf{u}_z)\otimes (\mathbf{M}_x+\mathbf{M}_z) & \mathbf{P}_z\otimes \mathbf{M}_z\\
-\mathbf{P}_z^T\otimes \mathbf{M}_z & \text{diag}(\mathbf{v}_z)\otimes (\mathbf{M}_y+\mathbf{M}_z) 
\end{bmatrix}^{-1} \begin{bmatrix}(\mathbf{P}_x^T\otimes \mathbf{M}_x)\text{vec}(\mathbf{X})\\(\mathbf{P}_y\otimes \mathbf{M}_y)\text{vec}(\mathbf{Y})
    \end{bmatrix}.
\end{align}
Based on the formula, we present a Floyd-type algorithm in Algorithm \ref{alg:1} by alternatively the optimization of transport plans $({\mathbf P}_x, {\mathbf P}_y, {\mathbf P}_z)$ and the optimization of anchors $({\mathbf Z}_x,\mathbf{Z}_y)$.

\subsection{Wasserstein distance as ground metric and estimation of transported points}
\label{wasslot}
This section we consider the inner cost matrix to be defined by the Wasserstein distances of the clustered distributions as $\mathbf{C}_z:=[{\cal{W}}^2_2(\mathbf{P}_x(\cdot|\mathbf{z}^x_m),\mathbf{P}_y(\cdot|\mathbf{z}^y_n))]_{m,n}$. This form arises when we represent the centroids by measures of points as $\tilde{\mathbf{z}}^x=\sum_{i=1}^N\mathbf{P}_x(\mathbf{x}_i|\mathbf{z}^x)\delta_{\mathbf{x}_i},$ $\tilde{\mathbf{z}}^y=\sum_{j=1}^M\mathbf{P}_y(\mathbf{y}_j|\mathbf{z}^y)\delta_{\mathbf{x}_j}$. Intuitively, $\tilde{\mathbf{z}}^x$ represents the distribution of the source points associated with the anchor $\mathbf{z}^x$ while $\tilde{\mathbf{z}}^y$ represents that of the target points associated with the anchor $\mathbf{z}^y$. We thus measure the distance between the anchors using the minimal transportation cost between the two distributions, which is then equivalent to the Wasserstein distance between the two measures. The additional challenges in optimizing \alg here is that the inner cost matrix depends on the transport plans of outer OTs. We propose to use an alternating optimization scheme that optimize transports plans $(\mathbf{P}_x,\mathbf{P}_y,\mathbf{P}_z)$ while keeping fixed $\mathbf{C}_{z}$ and then optimize $\mathbf{C}_{z}$ while keeping $(\mathbf{P}_x,\mathbf{P}_y,\mathbf{P}_z)$ fixed. However, the computation complexity of $\mathbf{C}_{z}$ could be high as it requires solving $k_xk_y$ OTs per iteration, where $k_x$, $k_y$ are the numbers of anchors in the source and target, respectively. We can reduce the computation by only solving a few OTs corresponding to anchor pairs $(m,n)\in [k_x]\times [k_y]$ that have high transportation probabilities $\mathbf{P}_z$. Specifically, for each $(m,n)\in [k_x]\times[k_y]$, we adopt the criteria of only calculating ${\cal{W}}^2_2(\mathbf{P}_x(\cdot|\mathbf{z}^x_m),\mathbf{P}_y(\cdot|\mathbf{z}^y_n))$ when $\mathbf{P}_z(\mathbf{z}^x_m,\mathbf{z}^y_n) > \theta\max_{n}\mathbf{P}_z(\mathbf{z}^x_m,\mathbf{z}^y_n)$, where $\theta$ is some predefined threshold (we use $\theta=0.5$ in the experiment). For the uncalculated Wasserstein distances, we set the costs to be infinity. We summarize the pseudo-code in algorithm \ref{alg:wa}, where \textsc{{SolveOT}} represents any OT solver, e.g. \cite{cuturi2013sinkhorn}, and \textsc{{UpdatePlan}} is the subroutine in Algorithm \ref{alg:1}.
The side-product $\mathbf{P}_{in}$ is a tensor that allows us to do accurate point to point alignment. Based on it, we propose the following estimated transportation of $\x_i$,
\begin{align}
    \hat{x}_i=\mathbf{P}_{in}[m^*,n^*,i,:]\mathbf{Y},~\text{where } m^*= \argmax_m \mathbf{P}_x[i,m], \text{and }n^*= \argmax_n \mathbf{P}_z[m^*,n].
\end{align}
\begin{algorithm}[t]
\begin{algorithmic}[1]
\caption{Latent Optimal Transport - Wasserstein Ground Metric (\url{LOT-WA})}\label{alg:wa}
\INPUT{$\mathbf{P}_x,\mathbf{P}_y,\mathbf{P}_z,\mathbf{P}_{in},\mathbf{K}_x,\mathbf{K}_y,\theta$.} 
\vspace{0.5mm}
\OUTPUT{$\mathbf{P}_x,\mathbf{P}_y,\mathbf{P}_z,\mathbf{P}_{in},\mathbf{Z}_x,\mathbf{Z}_y$}
\vspace{0.5mm}
\WHILE{not converging}
\FOR{$m=1,\dots,k_x$}
\FOR{$n=1,\dots,k_y$}
\IF{$\mathbf{P}_z(\mathbf{z}^x_m,\mathbf{z}^y_n) > \theta\max_{n}\mathbf{P}_z(\mathbf{z}^x_m,\mathbf{z}^y_n)$}
\STATE{$\mathbf{C}_{z}[m,n],\mathbf{P}_{in}[m,n] ={\color{blue}{\textsc{ SolveOT}}}\left({\mathbf{P}_x(\cdot|\mathbf{z}^x_m),\mathbf{P}_y(\cdot|\mathbf{z}^y_n)}\right)$} \hfill\COMMENT{//calculate ${\cal{W}}^2_2(\mathbf{P}_x(\cdot|\mathbf{z}^x_m),\mathbf{P}_y(\cdot|\mathbf{z}^y_n))$}
\ELSE
\STATE{$\mathbf{C}_{z}[m,n]=\infty$}
\ENDIF
\ENDFOR
\ENDFOR
\STATE{$\mathbf{P}_x, \mathbf{P}_y, \mathbf{P}_z\leftarrow {{\color{blue}{\textsc{ UpdatePlan}}}}{\left(\mathbf{K}_x,\mathbf{K}_y,\exp(-\mathbf{C}_{z}/\varepsilon)\right)}$}
\ENDWHILE
\end{algorithmic}
\end{algorithm}

\section{Variants of \alg}

\label{subsubsec:variant}
In this section, we discuss some natural extensions of \alg and their applications.
\paragraph{1) Data alignment under a global transformation:}
In many alignment applications, it is often the case that the learned features of the source and target are subject to some shift or transformation \cite{alvarez2019towards}. For example, in neuroscience \cite{lee2019hierarchical}, subspaces of the data clusters from neural and motor activities are subject to transformations preserving their principal angles. Formally, this is equivalent to saying that a transformation from the Stiefel manifold ${\cal{V}}_{d}:=\{\mathbf{O}\in\mathbb{R}^{d\times d}:\mathbf{O}^T\mathbf{O}=\mathbf{I}_d\}$ exists between them. Using the squared $\ell_2$-distance as the ground truth metric, we can thus pose an OT problem under this transformation as,
\begin{align}\label{transot}
    \min_{\mathbf{M}\in {\cal{V}}_{d}}\min_{{{ \mathbf{P}\mathbbm{1}} = \mu, { \mathbf{P}}^T{\mathbbm{1}} =\nu}}\sum_{i=1}^{N}\sum_{j=1}^M\|\mathbf{M}\x_i-\y_j\|_2^2\mathbf{P}_{i,j},
\end{align}
where $d$ is the dimension of the data. Observe that (\ref{transot}) is a sum over the number of all data points (with an order of ${\cal{O}}(NM)$), it could be fragile to outliers or noise. On the other hand, \alg could remedy the issue by casting an optimization only on the anchors $\mathbf{z}^x$, $\mathbf{z}^y$. Since the anchors are some averaged representations of data points, they are smooth and robust to data perturbations.
The resultant optimization is,
\begin{align}\label{translot}
    \min_{\mathbf{M}\in {\cal{V}}_{d}}\min_{{{ \mathbf{P}_z\mathbbm{1}} = \mu_z, { \mathbf{P}_z}^T{\mathbbm{1}} =\nu_z},\mathbf{P}_x,\mathbf{P}_y}\langle \mathbf{C}_x,\mathbf{P}_x \rangle + {\color{blue}\sum_{i=1}^{k_1}\sum_{j=1}^{k_2}\|\mathbf{M}\mathbf{z}^x_i-\mathbf{z}^y_j\|_2^2(\mathbf{P}_z)_{i,j}}+\langle \mathbf{C}_y,\mathbf{P}_y\rangle.
\end{align}
Now the order of the number of terms in the objective is greatly reduced to ${\cal{O}}(k_1k_2)$. As the transformation is cast upon the anchors, the solution tends to be more robust to perturbations.

Below we propose a simple procedure to optimize $\mathbf{M}$. Our method is based on an alternating strategy, where we alternate between two procedures: (1) $\mathbf{P}_x,\mathbf{P}_z,\mathbf{P}_y$ are fixed when we update $\mathbf{M}$; (2) $\mathbf{M}$ is fixed when we update $\mathbf{P}_x,\mathbf{P}_z,\mathbf{P}_y$. The later is nothing more than the algorithm \ref{alg:1} with $\mathbf{C}_z := [\|\mathbf{M}\mathbf{z}^x_i-\mathbf{z}^y_j\|_2^2]_{i,j}$. Hence we would focus on the procedure (1). Given a fixed $\mathbf{P}_z^*$, the minimization of the middle term of (\ref{translot}), after simplification, is equivalent to,
\begin{align}\label{transopt}
    \max_{\mathbf{M}\in {\cal{V}}_{d}} \langle \mathbf{M}\mathbf{Z}_x{\mathbf{P}_z^*},\mathbf{Z}_y\rangle.
\end{align}
The problem admits an elegant solution according to the following lemma.
\begin{lemma} (Lemma 3.1  \cite{alvarez2019towards})
Let $\mathbf{U}\mathbf{\Sigma}\mathbf{V}^T$ be the SVD decompostion of $\mathbf{Z}_y(\mathbf{Z}_x{\mathbf{P}_z^*})^T$, then $\mathbf{M}^*=\mathbf{U}\mathbf{V}^T$ is the optimal solution to (\ref{transopt}). 
\end{lemma}
Finally, we note that the class of transformations depends on the application, and the
applicability of \alg could go beyond the Stiefel manifold considered here. 

\begin{algorithm}[t]
\begin{algorithmic}[1]
\caption{Unbalanced Latent Optimal Transport }\label{alg:unlot}
\INPUT{$\mathbf{K}_x,\mathbf{K}_y,\mathbf{K}_z$,$\tau_1,\tau_2,\epsilon$.} 
\vspace{0.5mm}
\OUTPUT{$\mathbf{P}_x, \mathbf{P}_y, \mathbf{P}_z, \mathbf{u}_z, \mathbf{v}_z,{\z}_1,{\z}_2$}
\vspace{0.5mm}
\ENSURE{ $\alpha_x\leftarrow \mathbbm{1}_N,\beta_x\leftarrow \mathbbm{1}_{k_1},\alpha_y\leftarrow \mathbbm{1}_{k_2},\beta_y\leftarrow \mathbbm{1}_M,\alpha_z\leftarrow \mathbbm{1}_{k_1},\beta_z\leftarrow \mathbbm{1}_{k_2}$, ${\z}_1\leftarrow \mu, {\z}_2\leftarrow \nu$  }
\vspace{0.5mm}
\WHILE{not converging}
\STATE{$({\z}_1,\alpha_x)\leftarrow \left((\alpha_x\odot \mathbf{K}_x\beta_x)^{\frac{\epsilon}{\tau_1+\epsilon}}\odot {\z}_1^{\frac{\tau_1}{\tau_1+\epsilon}},~~~\alpha_x^{\frac{\epsilon}{\tau_1+\epsilon}}\odot ({\z}_1\oslash \mathbf{K}_x\beta_x)^{\frac{\tau_1}{\tau_1+\epsilon}}\right)$}
\vspace{0.5mm}
\STATE{$({\z}_2,\beta_y)\leftarrow \left((\beta_y\odot \mathbf{K}_y^T\alpha_y)^{\frac{\epsilon}{\tau_2+\epsilon}}\odot {\z}_2^{\frac{\tau_2}{\tau_2+\epsilon}},~~~\beta_y^{\frac{\epsilon}{\tau_2+\epsilon}}\odot ({\z}_2\oslash \mathbf{K}_y^T\alpha_x)^{\frac{\tau_2}{\tau_2+\epsilon}}\right)$}
\vspace{0.5mm}
\STATE{$\mathbf{u}_{z} \leftarrow ((\alpha_z\odot \mathbf{K}_z\beta_z)\odot(\beta_x\odot \mathbf{K}_x^T\alpha_x))^{\frac{1}{2}}$}
\vspace{0.5mm}
\STATE{$\beta_x\leftarrow \mathbf{u}_z\oslash \mathbf{K}_x^T\alpha_x$; $\alpha_z\leftarrow \mathbf{u}_z\oslash \mathbf{K}_z\beta_z$} 
\vspace{0.5mm}
\STATE{$\mathbf{v}_z \leftarrow ((\alpha_y\odot \mathbf{K}_y\beta_y)\odot(\beta_z\odot \mathbf{K}_z^T\alpha_z))^{\frac{1}{2}}$}
\vspace{1mm}
\STATE{$\beta_z\leftarrow w_{zy}\oslash \mathbf{K}_z^T\alpha_z$; $\alpha_y\leftarrow \mathbf{v}_z\oslash K_y\beta_y$} 
\ENDWHILE
\vspace{0.5mm}
\REQUIRE{$\mathbf{P}_x={\bf D}(\alpha_x)\mathbf{K}_x{\bf D}(\beta_x),\mathbf{P}_y={\bf D}(\alpha_y)\mathbf{K}_y{\bf D}(\beta_y),\mathbf{P}_z={\bf D}(\alpha_z)\mathbf{K}_z{\bf D}(\beta_z),\mathbf{u}_z,\mathbf{v}_z,{\z}_1,{\z}_2$} 
\end{algorithmic}
\end{algorithm}

\paragraph{2) Unbalanced variant of \alg\hspace{-1mm}:}
There are often cases in practice where the distribution of data is not normalized so the total mass of the source and target are not equal. In this case, OT preserving marginal distributions does not admit a solution. To deal with this problem, the unbalanced optimal transport (UOT) is among one of the most prominent proposed method \cite{liero2018optimal}. Instead of putting hard constraints on the marginal distributions, UOT considers an optimization problem regularized by the deviation to the marginal distributions measured by the KL-divergence. Just as in OT, \alg can be naturally extended in this case by modifying the objective in Eqn.~(\ref{cost1for}) to,
\begin{align}\label{opt:unlot}
    &\min_{\mathbf{P}_x,\mathbf{P}_y,\mathbf{P}_z,\mathbf{z}_1,\mathbf{z}_2} \epsilon \text{KL}(\mathbf{P}_x|\mathbf{K}_x) + \tau_1 \text{KL}(\mathbf{z}_1|\mu) + \epsilon \text{KL}(\mathbf{P}_z|\mathbf{K}_z) + \epsilon \text{KL}(\mathbf{P}_y|\mathbf{K}_y) + \tau_2\text{KL}(\mathbf{z}_2|\nu) \\
    &\text{subject to: } \mathbf{P}_x\mathbbm{1} =\mathbf{z}_1,\mathbf{P}^T_x\mathbbm{1}=\mathbf{u}_z,\mathbf{P}_z\mathbbm{1} = \mathbf{u}_z,\mathbf{P}_z^T\mathbbm{1} =\mathbf{v}_z,\mathbf{P}_y\mathbbm{1}=\mathbf{v}_z,\mathbf{P}^T_y\mathbbm{1}=\mathbf{z}_2. \nonumber
\end{align}
The algorithm could be derived similarly as in Appendix \ref{der1}, and we provide the pseudocode in Algorithm \ref{alg:unlot}.

\paragraph{3) Barycenter of hubs:}
Finding the barycenter is a classical problem of summarizing several data points with a single  ``mean''. In optimal transport, this amounts to the Wasserstein barycenter problem \cite{cuturi2014fast}. \alg casts an unique variant of the problem in the following scenario. Consider a shipping company ships items worldwide. The company has several hubs in its country to gather the items and then ship them to other countries. On the other hand, each country also has its own receiving hubs to receive the items. The company then must decide the locations of its hubs in its country to minimize the transport effort. The above scenario poses an unique barycenter problem in terms of optimizing locations of the hubs. By using an analogy of hubs to anchors, we can transform the problem into the framework of \alg, where one seeks the optimal sources' anchors $Z_x$ by solving the optimization,
\begin{align}\label{barycenter}
    \min_{{Z_x}}\min_{Z_{{y_i},i\in[K]}}\sum_{k=1}^K{\rm OT^L}(\mu,\nu_k).
\end{align}
Note the inner optimization is related to the hubs of the receiving countries, which depending on the applications, could be either fixed or free to optimize. 
To solve this problem, one can simply extend Algorithm \ref{alg:1} to
make it distributed. This is done by applying Algorithm \ref{alg:1} to each term in the summation.

\section{Additional experiments}


In this section, we describe the results for additional experiments that carried out to test \alg.
\subsection{Domain adaptation task}
\label{sec:DA_appendix}
\paragraph{Neural network implementation details:}
The classifier is a multi-layer perceptron with two 256-units hidden layers and ReLU activation functions. We use dropout to regularize the network. $28\times28$ Images are first normalized using $\text{mean}=0.1307$ and $\text{std}=0.3081$, then flattened to 784 dimensions before being fed to the network, which then outputs a 10-d logits vector. The classifier is trained using a cross-entropy loss. To get the predicted class, we used the argmax operator. 
The network only learns from the training samples of MNIST, and then its weights are frozen.
USPS images are smaller with $16\times16$ pixels. To feed them to our classifier, we apply zero-padding to get the desired input size.

\paragraph{Domain shift in neural networks experiments:}
For our experiment, we use 1000 randomly sampled images from each of these sets: the training and testing sets from MNIST and the testing set from USPS. The sampling isn't done in any particular way, so there is a different number of samples for each class in each set. For all experiments using the classifier, we use the argmax operator on the transported features to get the predictions after aligning source and target. We then compare to the ground truth labels to get the accuracy. All results can be found in Table~\ref{table:table1_extended}.

For the first experiment (Table~\ref{table:table1_extended}-DA), we align the testing set of USPS with the training set of MNIST (Figure~\ref{fig:samples}a). For the second and third experiments (Table~\ref{table:table1_extended}-D, DU), we study the influence of perturbations on the output space of the classifier. So we align a perturbed version of the testing set of MNIST with the training set of MNIST. For both experiments, we use coarse dropout, which sets rectangular areas within images to zero. These areas have large sizes. We generate a coarse dropout mask once and apply it to all images in our testing set (Figure~\ref{fig:samples}b). In addition to coarse dropout, we also eliminate some classes (2, 4, and 8 digit classes) to get the unbalanced case for the third experiment (Table~\ref{table:table1_extended}-DU).

\begin{figure}[t!]
    \centering
    \subfloat[][\footnotesize{MNIST-USPS Samples}]{\includegraphics[height=0.2\textwidth]{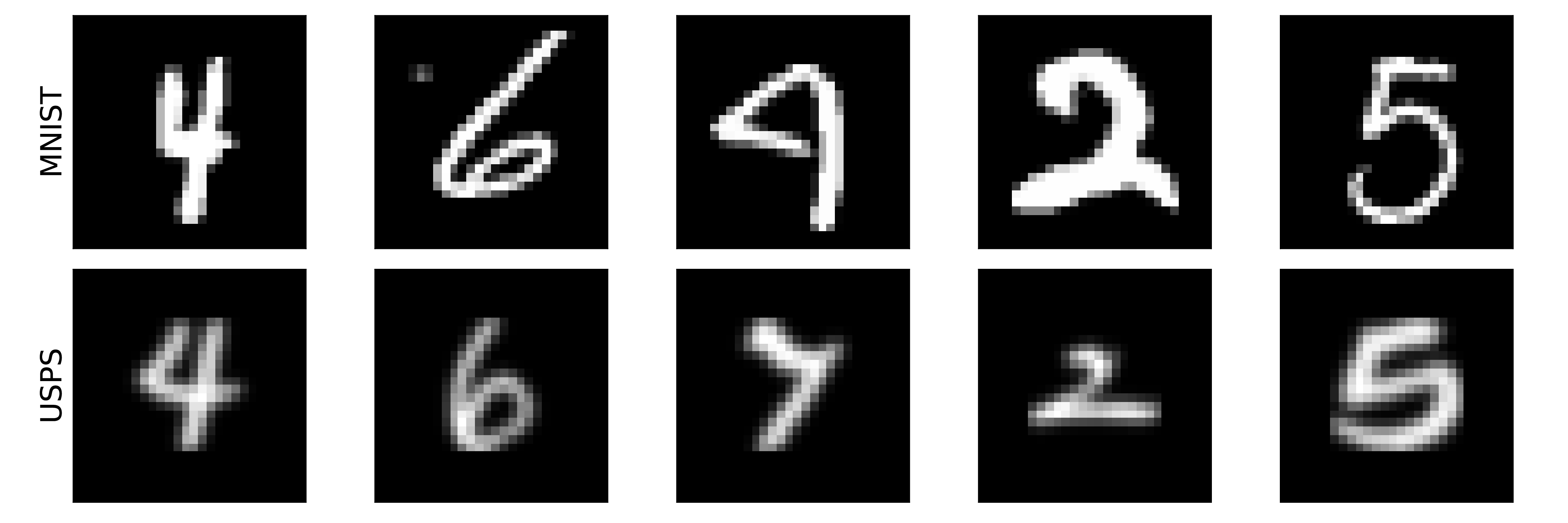}}
    \subfloat[][\footnotesize{Dropout Mask}]{\includegraphics[height=0.2\textwidth]{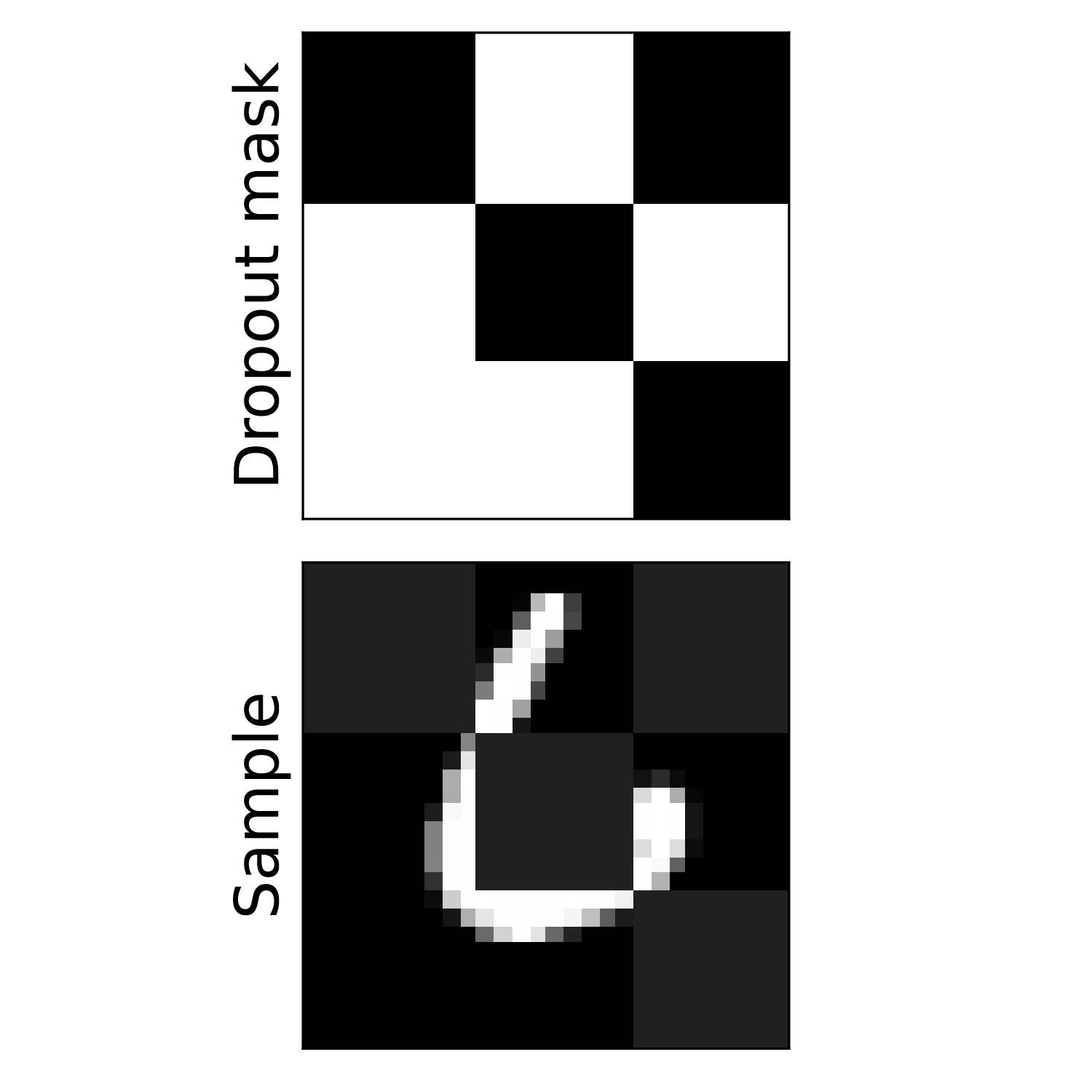}}
   \caption{{\em Samples used in deep neural network experiments.} (a) MNIST samples vs. USPS samples (after zero-padding), used in the Domain Adaptation experiment. (b) Coarse dropout mask applied over all samples in MNIST's testing set and a sample of the perturbed images, used in the domain shift experiment.}
   \label{fig:samples}
   \end{figure}

\begin{table}[ht]
\begin{center}
\begin{tabular}{l|c|cc|cc}
         & MNIST-USPS (DA) & \multicolumn{2}{c|}{MNIST (D)} & \multicolumn{2}{c}{MNIST (DU)}         \\
         & Accuracy        & Accuracy       & L2 error      & Accuracy      & L2 error               \\ \hline
Original & 79.3            & 65.7           & 0.74          & 72.6          & 0.72                   \\
\ot       & 76.9            & 73.4           & 0.59          & 61.5          & 0.71                   \\
\url{kOT}      & 79.4            & 74.0           & 0.53          & 60.9          & 0.73                   \\
SA       & 81.3            & 64.9           & -             & 72.3          & -                      \\
\fc       & 84.1            & 77.6           & \textbf{0.51} & 67.2          & 0.59                   \\
\lot      & \textbf{86.2}   & \textbf{78.2}  & 0.53 & \textbf{77.7} & \textbf{0.56}
\end{tabular}
\caption{{\footnotesize {\bf \em{Results for concept drift and domain adaptation for handwritten digits.}}  The classification accuracy and L2-error (transported samples vs. ground truth test samples) are computed for the synthetic drift experiment: Coarse dropout (left) and for the domain adaptation experiment: MNIST to USPS (right). Our method is compared with the accuracy before alignment (Original), entropy-regularized OT, k-means plus OT (\url{kOT}), and subspace alignment (SA).
\label{table:table1_extended}}}
\end{center}
\end{table}

\paragraph{Visualization of transported distributions:}
In Figure~\ref{fig:mnist}a, we project the distribution of the neural network's output features (for each set) in 2D using Isomap fit to the MNIST training set's output features. We use 50 neighbor points in the Isomap algorithm to get a reasonable estimate of the geodesic distance, as we have $10$ well-separated clusters in the output space of the deep neural network in the case of MNIST's training set. To ensure consistency, we use the projection learned on the MNIST training set with all other sets.

\paragraph{Visualization of transport plans:}
We provide further visualization of the transport plans obtained by \alg and \fc, for the Domain Adaptation experiment and the unbalanced alignment experiment (Figure~\ref{fig:allsankey}). 

\begin{figure*}[t!]

    \begin{subfigure}[t]{\textwidth}
    \centering
    \subfloat{\includegraphics[height=1.1\textwidth]
   {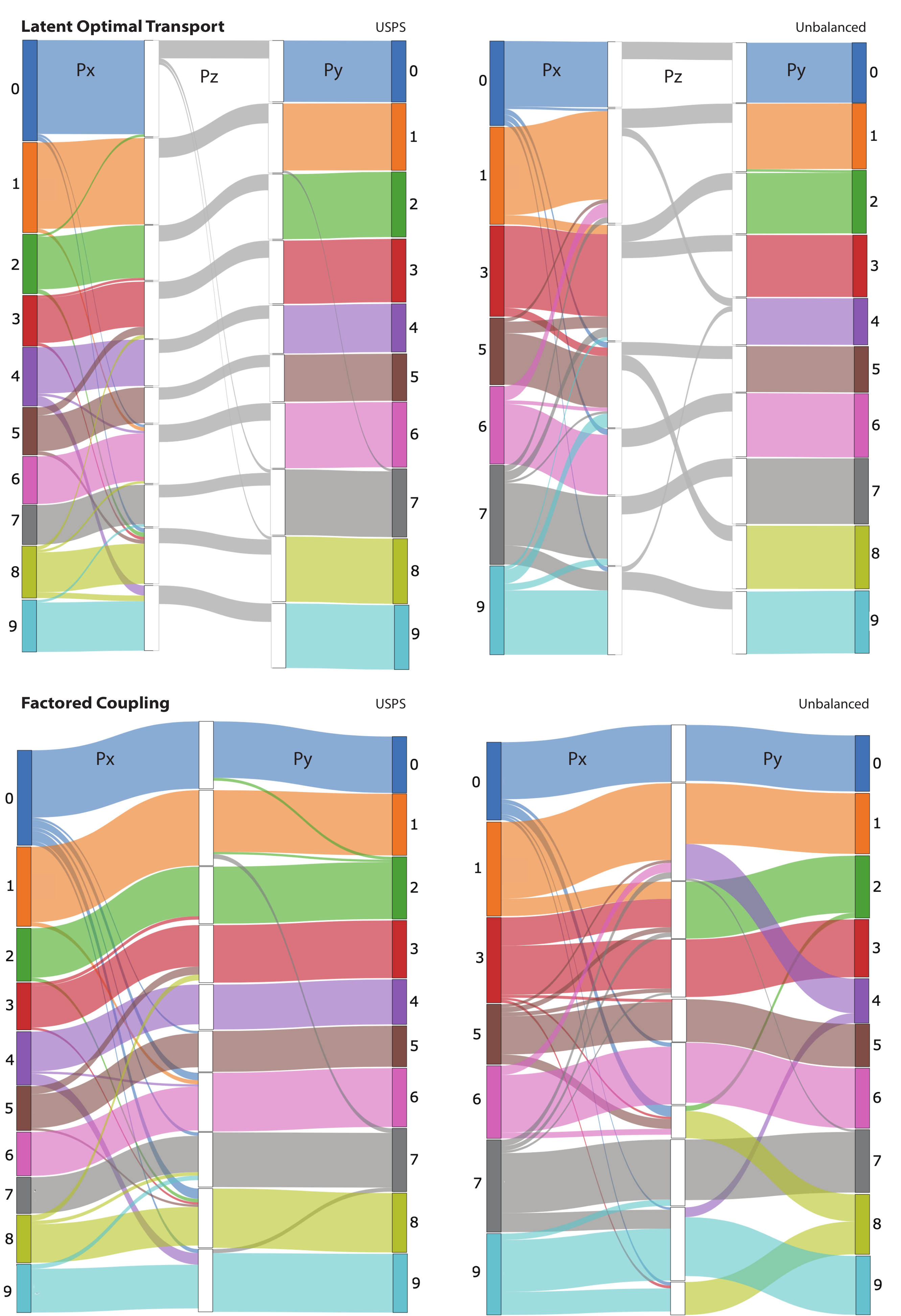}}
   \end{subfigure}
   
    \caption{\footnotesize {\bf{\em Visualization of transport plans obtained with \alg and factored coupling.}} We show the transport plans for \alg (top) and \fc (bottom) for the USPS (left) and unbalanced dropout example (right). \label{fig:allsankey}}
\end{figure*}

\paragraph{Visualization of transported points:} To get a better understanding of how the embedding of each sample is changing after transportation, we provide visualizations of where samples are being transported to in the target space. To describe the target space locally, we find the five nearest neighbors of a sample according to the L2 distance between the transported features of the source sample and the features of the target samples. In Figure~\ref{fig:neigh}, we show some of the cases where the three methods, \ot, \fc and \alg, disagree.

In the top two rows (Figure~\ref{fig:neigh}a-b), we see cases where \alg outperforms \ot and \fc. When all neighbors have the same labels, we can safely assume that we are not in a boundary between classes but deep within a class cluster, so this reflects on the confidence these methods have in their transportation. Both \ot and \fc confidently map the sample to the wrong region of the space (Figure~\ref{fig:neigh}a). 

In the bottom two rows (Figure~\ref{fig:neigh}c-d), we see more difficult cases. In (Figure~\ref{fig:neigh}c), we see that even though the closest neighbor (4 in light green) doesn't correspond to the ground truth label (7), \lot seems to be mapping the point to the boundary between classes 7 and 4, and it does in fact classify it correctly. Lastly, in (d), we see that none of the methods is able to recover information lost due to dropout.

\begin{figure}[t!]
 \begin{subfigure}[t]{\columnwidth}
        \centering
    \subfloat[][]{\includegraphics[width=1.0\textwidth]{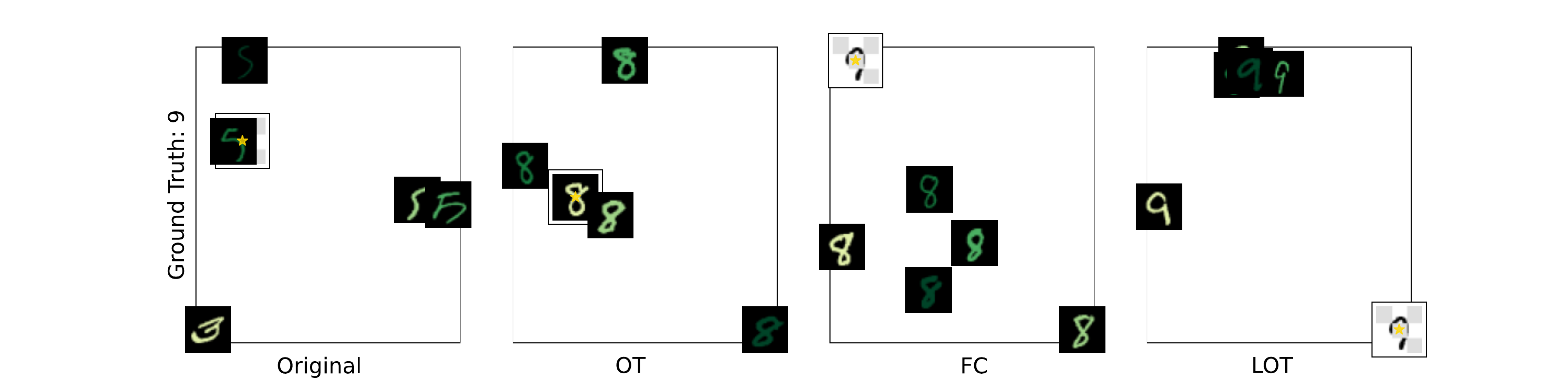}}
    \end{subfigure}
    
     \begin{subfigure}[t]{\columnwidth}
        \centering
   \subfloat[][]{\includegraphics[width=1.0\textwidth]{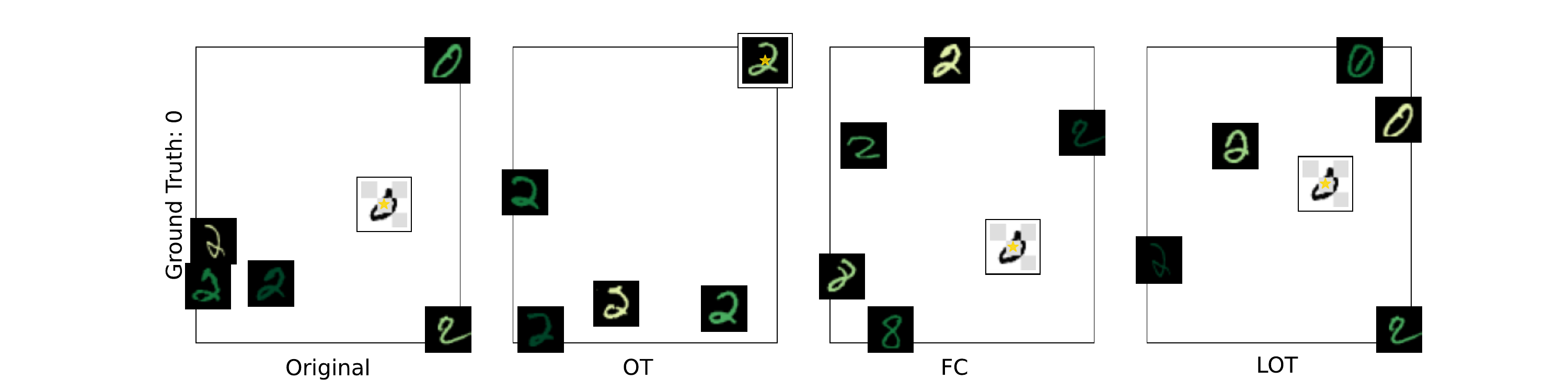}}
   \end{subfigure}
   
    \begin{subfigure}[t]{\columnwidth}
        \centering
   \subfloat[][]{\includegraphics[width=1.0\textwidth]{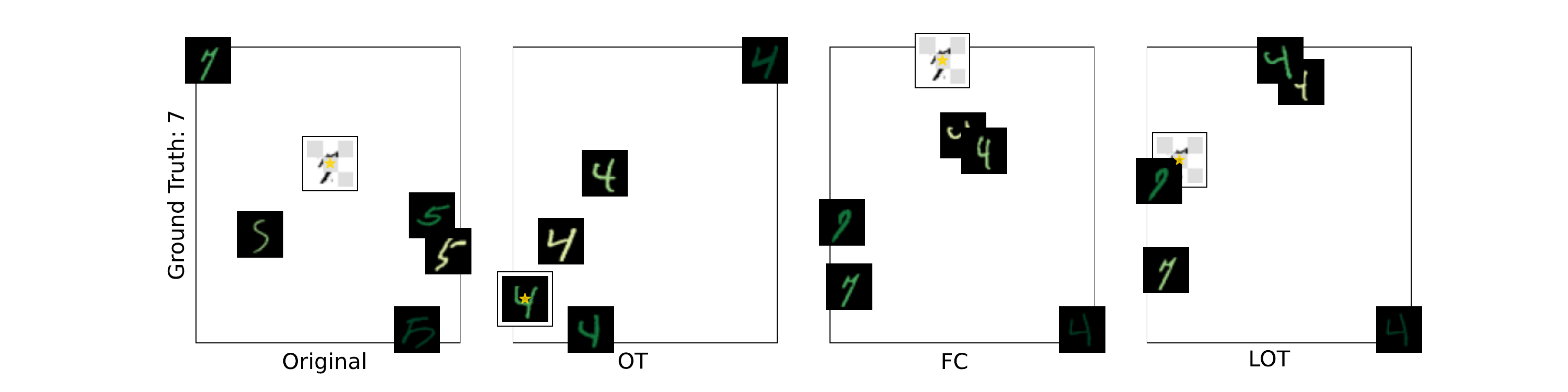}}
   \end{subfigure}
   
    \begin{subfigure}[t]{\columnwidth}
        \centering
   \subfloat[][]{\includegraphics[width=1.0\textwidth]{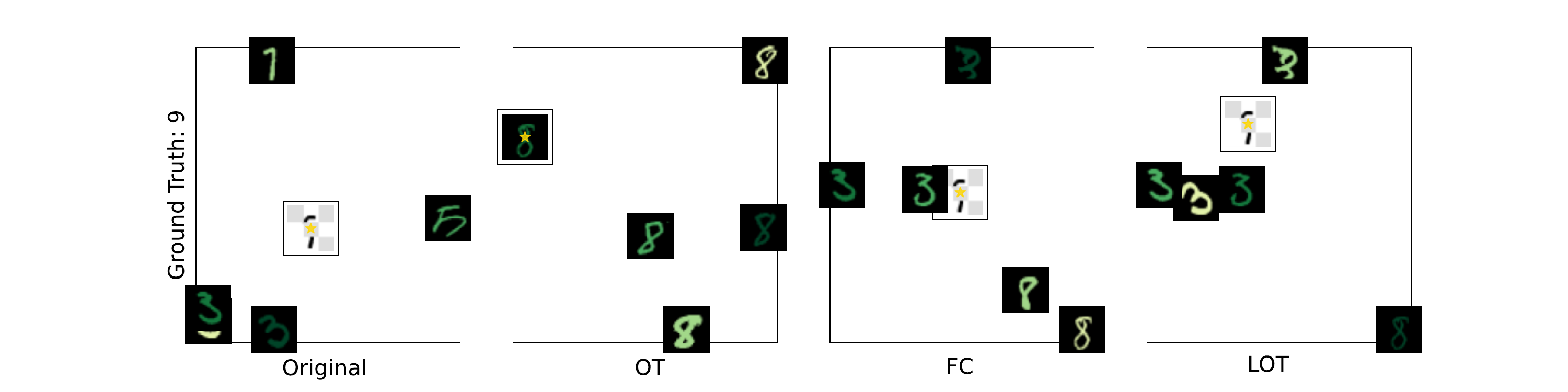}}
   \end{subfigure}
   \caption{{\bf \em Neighborhood change after transportation (DU)} For the unbalanced transport experiment, we show for multiple perturbed samples (white, marked with a $\star$) the change in the closest neighbors landscape before (Original) and after the transportation, for different methods (\ot, \fc and \alg). We find the five nearest neighbors (black) according to the L2 distance between the transported output features of the sample and the features from the target training samples, and arrange the samples in 2D space using the Isomap projection. We apply a gradient filter over the neighbors' images such that the closest neighbor is light green.}
   \label{fig:neigh}
\end{figure}

\begin{figure}[t!]
       \centering
   \subfloat{\includegraphics[width=1\textwidth]{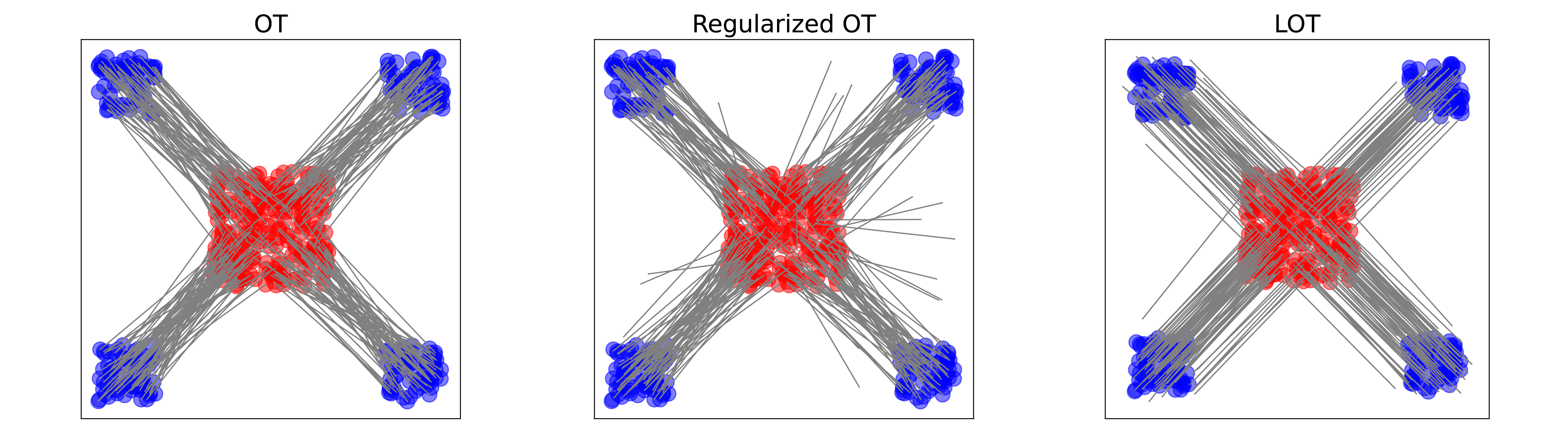}}

\caption{\label{fig:frag}\footnotesize 
{ \bf \em{Fragmented Hypercube.}} We visualize the estimated transports $\hat{x}=\int yp(y|x)dx$~ for {OT}, regularized {OT}, and \alg. 
}

\end{figure}

\begin{figure}[t!]
       \centering
   \subfloat{\includegraphics[width=1\textwidth]{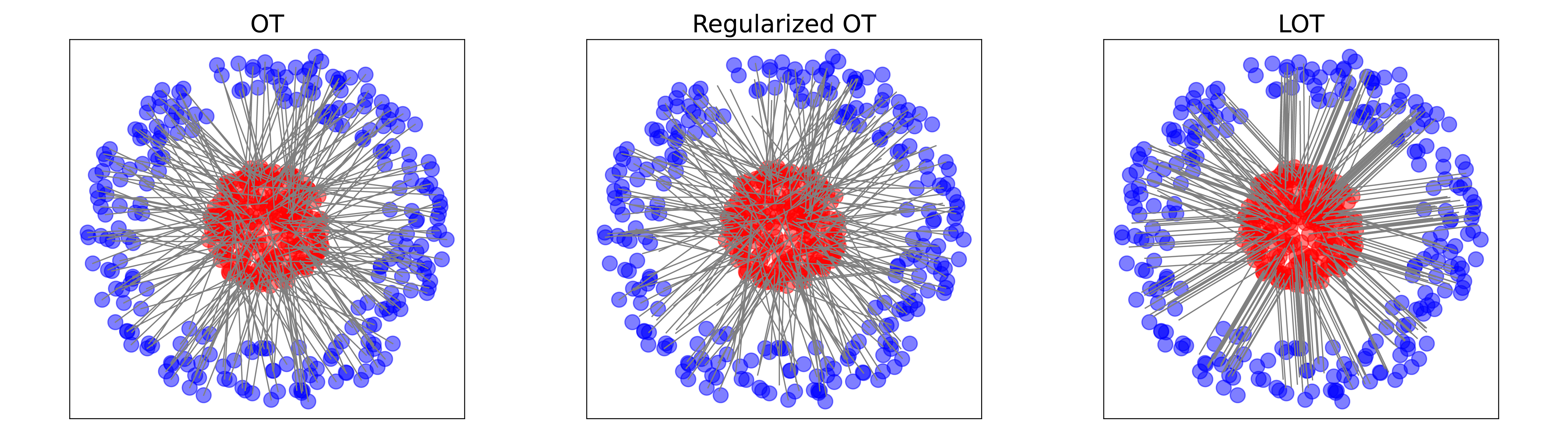}}
\caption{\footnotesize 
{\label{fig:annulus} \bf \em{Disk to Annulus.}} We visualize the estimated transports $\hat{x}=\int yp(y|x)dx$~ for {OT}, regularized OT, and \alg\hspace{-1mm}. 
}

\label{fig:map_unequal_illus}
\end{figure}

\subsection{Additional synthetic experiments}
We study the behaviors of \alg for two synthetic experiments performed in \cite{forrow2019statistical,paty2019subspace}.

\paragraph{Fragmented Hypercube:} In the first experiment, the source distribution is $\mu = {\cal{U}}([-1,1]^d)$ and the target distribution is $\nu = T_{\#}\mu$ (i.e., $\nu\overset{d}{=} T(x), x\overset{d}{=}\mu$), where $T(x)= x + 2\text{sign}(x)\odot(e_1+e_2)$, and $e_i$ is the canonical basis of $\mathbf{R}^d$. In this example, the signal lies only in the first $2$ dimensions while the remaining $d-2$ dimensions are noises. The data shows explicit clustering structure.
We investigate the estimated transports $\hat{x}=\int yp(y|x)dx$ produced by OT, entropy- regularized OT, and \alg with $k_x=k_y=4$ to see if \alg can capture the data structure and whether it provides robust transport.
We choose $d=30$ and draw $N=250$ points from each distribution. The result is shown in Figure \ref{fig:frag}. We see that all methods capture the cluster structure, but both \ot and regularized are sensitive to noise in $d-2$ dimensions, while \alg
provides a data transport that is more robust against noise in high dimensional space.

\paragraph{Disk to annulus:} In the second experiment, we consider data without separate clustering structure.  
In this example, the source and target distributions are
\begin{align}
    \mu&={\cal{U}}\{\x\in \mathbb{R}^d: 0\leq \|(\x_1,\x_2)\|_2\leq 1,\x_i\in\{0,1\}, i =3,\dots,d\}\nonumber\\
    \nu&={\cal{U}}\{\x\in \mathbb{R}^d: 2\leq \|(\x_1,\x_2)\|_2\leq 3,\x_i\in\{0,1\}, i =3,\dots,d\}. \nonumber
\end{align}
Again only the first $2$ dimensions contain the signal while the rest $d-2$ are noise. We draw $250$ points from each distribution and choose $d =30$. We compare the estimated transports $\hat{x}=\int yp(y|x)dx$ for OT, entropy-regularized OT, and \alg with $k_x=k_y=15$. In this example, we increase the number of anchors as the annulus can be regarded as having infinite clusters. We visualize the result in Figure \ref{fig:annulus}. We observe that \alg is more regular than the other two, but the support of the target is not fully covered by the transportation. This is because the rank constraint imposed on \alg limits its degree of freedom to transport data.

\begin{figure*}
     \centering
     \subfloat[][Optimal Transport]{\includegraphics[width=.2\textwidth]{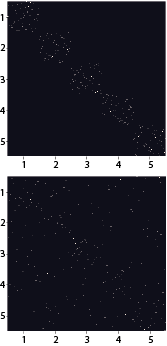}}
     \hspace{10mm}
     \subfloat[][\lot ]{\includegraphics[width=.2\textwidth]{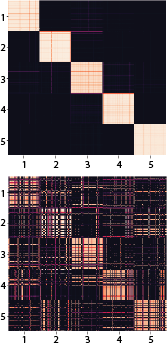}}
     \hspace{10mm}
     \subfloat[][Cross-Correlation]{\includegraphics[width=.2\textwidth]{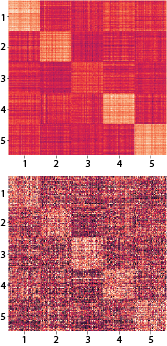}}
     \caption{\footnotesize{\em \bf Examining the connection between cluster coherence and transport.} In (a-b), we visualize the transport plan obtained by OT and \alg when the different mixture components in the model are weakly correlated (top) and strongly correlated (bottom). In (c), we examine the cross-correlation between points in the source and target for the two conditions.}
\label{fig:corre}
\vspace{-0mm}
\end{figure*}

\subsection{Examining the impact of cluster correlation on performance}
Next, we study the ability of \alg of capturing clustering structure as the correlation between clusters increases. It is observed in \cite{lee2019hierarchical} that equally spaced clusters (i.e., uniformly distributed subspace angles) are harder to align. In the spirit of measuring cluster similarity as in \cite{soltanolkotabi2012geometric}, we define the cross-correlation between two clusters ${{\cal{Y}}_i}$ and ${{\cal{Y}}_j}$ as $\mu_c({{\cal{Y}}_i},{{\cal{Y}}_j}):= \sum_{x\in{{\cal{Y}}_i},y\in {{\cal{Y}}_i}}\langle x,y\rangle/|{{\cal{Y}}_i}||{{\cal{Y}}_j}|$. In Fig.~\ref{fig:corre}, we plot the transport plans for GMM with $5$ clusters in two different correlation regimes: (i) the top row shows the regime where the cross-correlation is strong within the same class; (ii) the bottom row shows the uniform correlation case where clusters are approximately equally spaced and the correlations are much more uniform across classes. We observe in both results that \alg has transport plans similar to the cross-correlation in the pattern, while for uniform case, the transport plan of \ot seems less structured. This example suggests that \alg can be a useful tool for data visualization.

\section{Details of GMM Experiment (E2)}\label{sec:detailexp2}
   In this experiment, we use a Gaussian mixture model (GMM) with $M$ components in $\mathbb{R}^d$ to generate the data for the source and target. For each component, the mean is randomly generated from a random standard normal distribution, and the covariance matrix is generated from a Wishart distribution. To model low-dimensional latent structure, we then apply random $k$-dimensional projection and add a $d$-dimensional standard normal noise to each component. The source and target draw $100$ points from each component independently.
   
   Below we provide the details to each plot of (a) to (e). The default $k$ is set to 5. 
   \paragraph{Parameters of the data generation:}
   \begin{itemize}
       \item (a) We set $M=4, d=30$. We choose a random unit vector and rotate the mean of the component by $\theta$-degree using the vector as the axis.
       \item (b) $M=4, d= 30$. For any ratio $r$ of outliers, we randomly pick up $rn$ number of points and modify its distribution by a linear combination of Gaussian noise and the original point. The variance of Gaussian noise is equal to half the squared mean of the component.
       \item (c) $M=4$. We increase the dimension $d$. As the signal lies in a $k$-dimensional space, the rest of $d-k$-dimensions are noisy.
       \item (d) $d=30$. Here we generate a GMM with $10$ components. While the target has $10$ classes with $100$ points per class initially, we vary the number of the source class from $2$ to $10$. The $x$-axis denotes the ratio between the number of components to the source and to the target.
       \item (e) $d=30,M=4$. We set the number of the source and target anchors to be equal and increase the number from $2$ to $100$. Note that this number also upper bounds the rank of the associated transport plan.
      \end{itemize}
    \paragraph{Algorithmic implementation:}
    \begin{itemize}
        \item We use the POT: Python Optimal Transport package \cite{flamary2017pot} \href{https://pythonot.github.io/}{https://pythonot.github.io/} as the implementation for a vanilla optimal transport.
        \item In all of our GMM experiments, the entropy regularization parameter $\varepsilon$ is set to $10$.
        \item For the transport estimation, we use the expected 
        transportation defined by a transport plan, which is:
        \begin{align}
            &\hat{\mathbf{X}} =  \text{diag}(\mu)\mathbf{P} \mathbf{Z}_2, \text{ for \ot.}\label{otest}\\
            &\hat{\mathbf{X}} = \text{diag}(\mu)\mathbf{P}_x (\mathbf{Q}_y-\mathbf{Q}_x), \text{ for \fc \cite{forrow2019statistical}}\label{fcest}\\
            & \hat{\mathbf{X}}= \text{diag}(\mu^{-1})\mathbf{P}_x\text{diag}((\mathbf{P}_z\mathbf{1})^{-1})\mathbf{P}_z(\mathbf{Q}_y-\mathbf{Q}_x),\text{ for \alg (the estimator in Section \ref{sec:algorithm})},\label{lotest}
        \end{align}
        where $\mathbf{Q}_x = \text{diag}(\mu^{-1})\mathbf{P}_x^T\mathbf{X}^T$, $\mathbf{Q}_y = \text{diag}(\nu^{-1})\mathbf{P}_y\mathbf{Y}$ denote the centroids for \fc and \alg.
        \item We adopt the 1-NN classification rule where the class of a point from the source is predicted by the class of the nearest neighbor (among the target points) of its estimated transportation.
    \end{itemize}

\section{Choice of hyperparameters}\label{sec:hyptun}
\paragraph{Code availability:} We provide an implementation of \alg and a demo for a dropout experiment on MNIST in the supplementary file. We use the POT: Python Optimal Transport package \cite{flamary2017pot} \href{https://pythonot.github.io/}{https://pythonot.github.io/} as the implementation for a vanilla optimal transport.

\paragraph{Hyperparameter tuning:}
\alg has two main hyperparameters that must be specified: (i) the number of anchors and (ii) the entropy-regularization parameter epsilon. For domain adaptation experiments (E2-3), the number of anchors was set to $10$, using the prior we had about there being 10 classes of digits. The regularization parameter epsilon was set to $50$; this parameter depends on the scale of the data. We note that in the case of (E2), the standard deviation of the logit outputs of the network was around $10^3$. More generally, we proceed as follows when choosing the hyperparameters:

\begin{itemize}
    \item {\bf Number of anchors.}~ To optimize the number of anchors, one can use domain knowledge or model selection procedures common in clustering. For instance, in the case of MNIST-USPS, we know to expect $10$ clusters, each corresponding to a digit, and set the number of anchors to be at least this amount. In cases where we do not have a priori knowledge about the number of clusters in the data, we can use standard approaches for model selection in clustering (e.g., silhouette score \cite{ROUSSEEUW198753}, Calinski-Harabaz index \cite{calinski1974}). In our experiments, we find that the number of anchors can be increased progressively before observing a phase transition (or change point) where the accuracy increases significantly (see Figure~\ref{fig:gmm}.e). This change point typically coincides with the number of clusters in the data and can be used to select the number of anchors. Furthermore, while some interpretability may be sacrificed when we overestimate the number of anchors, we find that this does not hinder the performance of the method in terms of the quality of overall transport (see Figure~\ref{fig:gmm}.e). This is reminiscent of the elbow phenomena found with clustering methods, so we can imagine utilizing established clustering metrics for tuning these two parameters. This is left for future work, as the simple elbow method was found to be sufficient. 
    \item {\bf Entropy- regularization parameter.}~ The choice of the regularization parameter mainly depends on the scale of the data. In our experiments, we search for the lowest epsilon for which \alg converges using a simple Bisection method. This same approach is used for selecting epsilon for \ot in our experiments. We find that the regularization value is inherent to the data and not the task (Exp 2).
\end{itemize} 


\end{document}